%% file: main.tex
\documentclass{arxiv}

\usepackage{wrapfig}
\usepackage{times}
\usepackage{algorithm}
\usepackage{algorithmic}
\DeclareMathOperator*{\argmax}{argmax}
\DeclareMathOperator*{\argmin}{argmin}
\newcommand{\norm}[1]{\left\lVert#1\right\rVert}
\newcommand{\abs}[1]{\left\lvert#1\right\rvert}
\newcommand{\ts}{\textsuperscript}

\title[Inducing Local Lipschitzness for Robust GAIL]{On the Benefits of Inducing Local Lipschitzness for Robust Generative Adversarial Imitation Learning}

\author{%
 \Name{Farzan Memarian} \Email{farzan.memarian@utexas.edu}\\
 \addr NVIDIA Corporation
 \AND
 \Name{Abolfazl Hashemi} \Email{abolfazl@purdue.edu}\\
 \addr Purdue University%
  \AND
 \Name{Scott Niekum} \Email{sniekum@cs.umass.edu}\\
 \addr The University of Massachusetts Amherst %
  \AND
 \Name{Ufuk Topcu} \Email{utopcu@utexas.edu}\\
 \addr The University of Texas at Austin%
}

\begin{document}

\maketitle

\begin{abstract}
\input{sections/sec_abstract}
\end{abstract}
\vspace{-4mm}
\section{Introduction}
\label{sec:introduction}
\input{sections/sec_intro}


\section{Background}
\label{sec:background}
\input{sections/sec_back}

\vspace{-2mm}
\section{Insights on the Lipschitzness of the Discriminator and the Generator}
\label{sec:thm}

\input{sections/sec_theorem}

\vspace{-2mm}
\section{Lipschitz-Inducing Regularization for GAIL}
\label{sec:method}

\input{sections/sec_method}

\begin{figure}
\centering
\subfigure[\footnotesize Walker2d, reg. disc.]{\label{subplot:Walker2d-d} \includegraphics[width=0.24\textwidth]{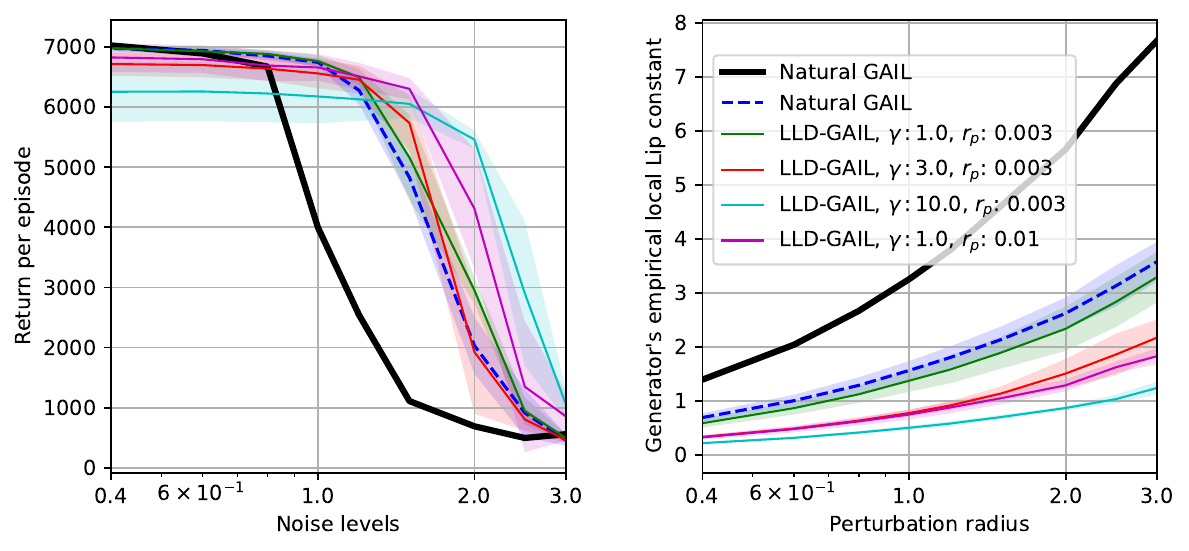}}
\hfill
\subfigure[\footnotesize Walker2d, reg. gen.]{\label{subplot:Walker2d-g} \includegraphics[width=0.24\textwidth]{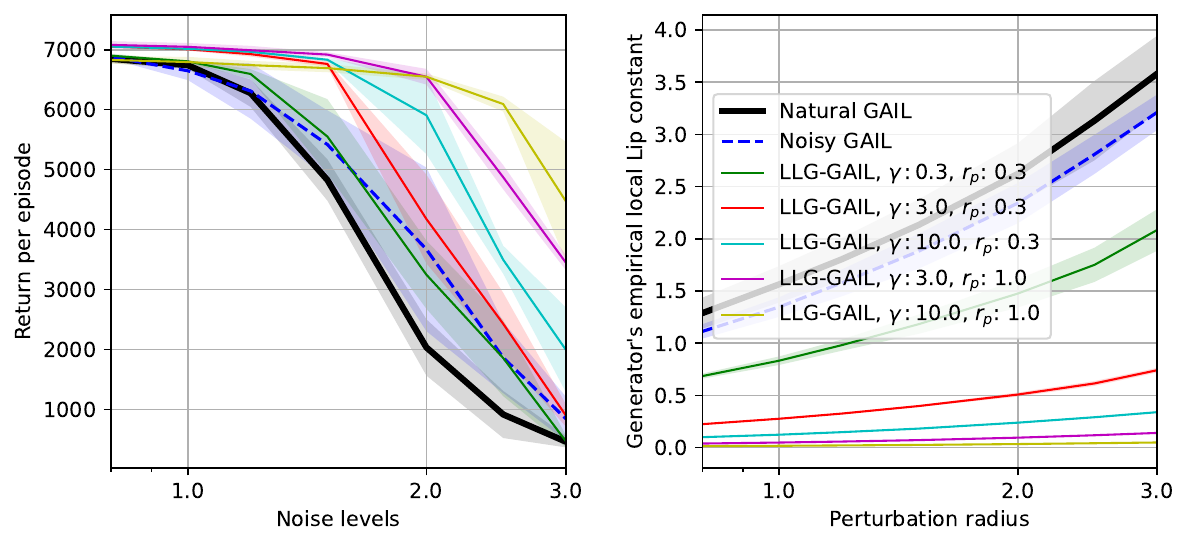}}%
\hfill
\subfigure[\footnotesize Hopper, reg. disc.]{\label{subplot:Hopper-d} \includegraphics[width=0.24\textwidth]{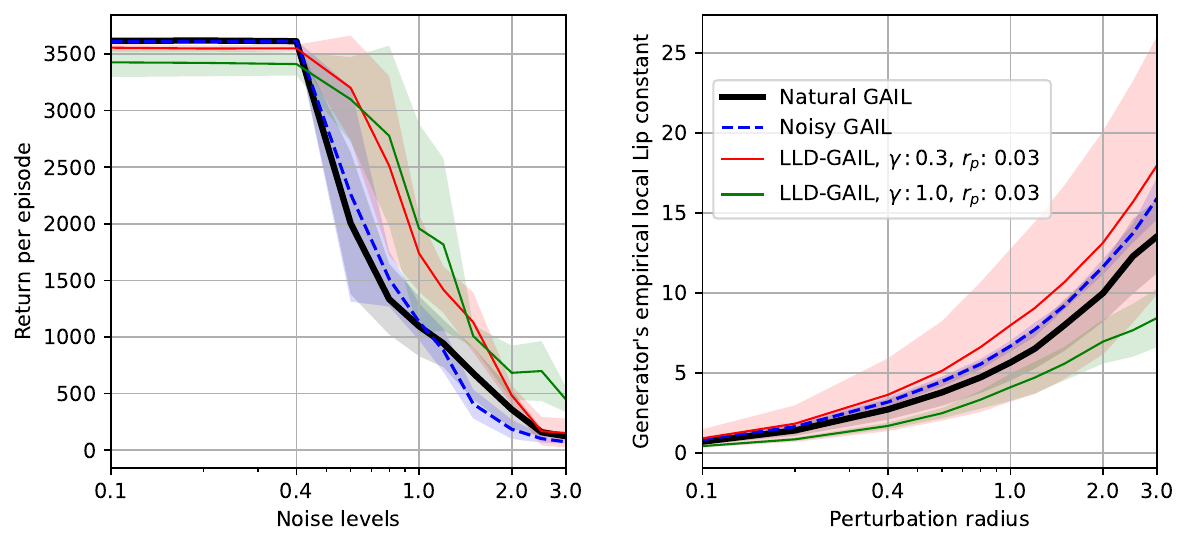}}%
\hfill
\subfigure[\footnotesize Hopper, reg. gen.]{\label{subplot:Hopper-g} \includegraphics[width=0.24\textwidth]{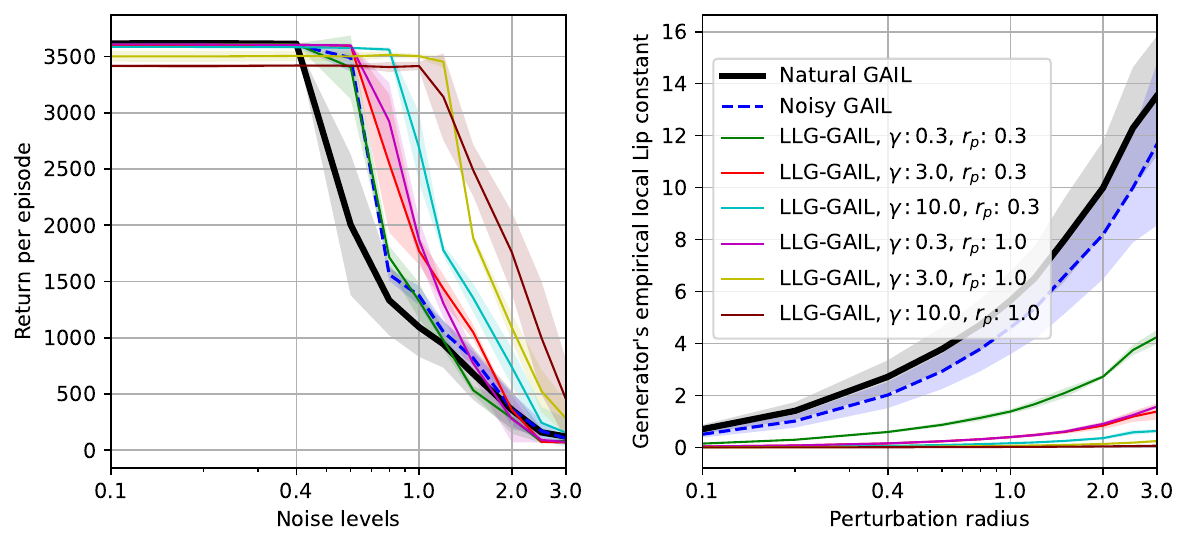}}%
\hfill
\subfigure[\footnotesize HalfCheetah, reg. disc.]{\label{subplot:HalfCheetah-d} \includegraphics[width=0.24\textwidth]{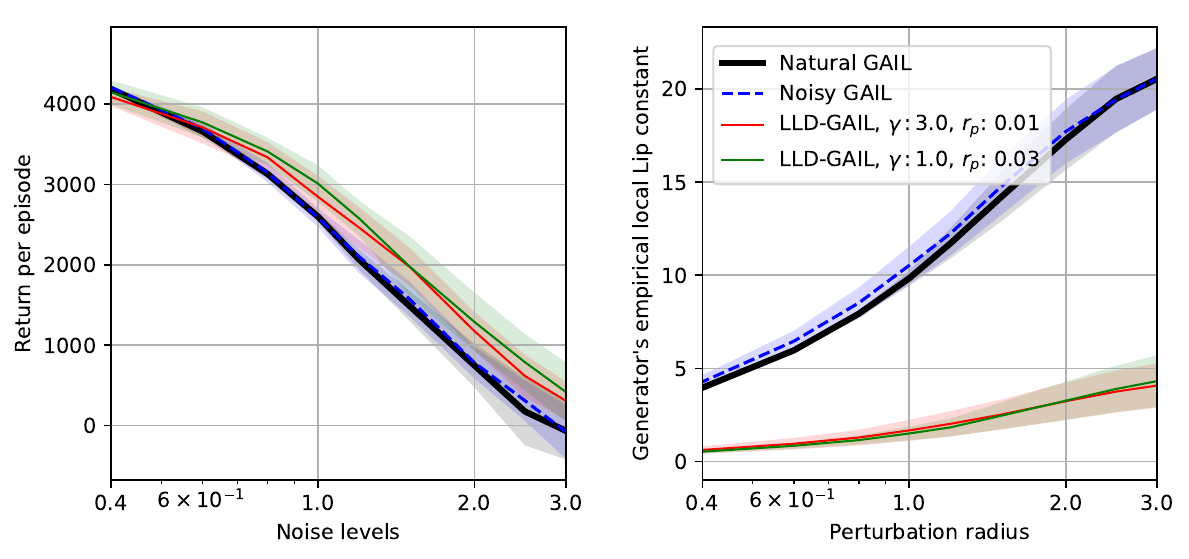}}%
\hfill
\subfigure[\footnotesize HalfCheetah, reg. gen.]{\label{subplot:HalfCheetah-g} \includegraphics[width=0.24\textwidth]{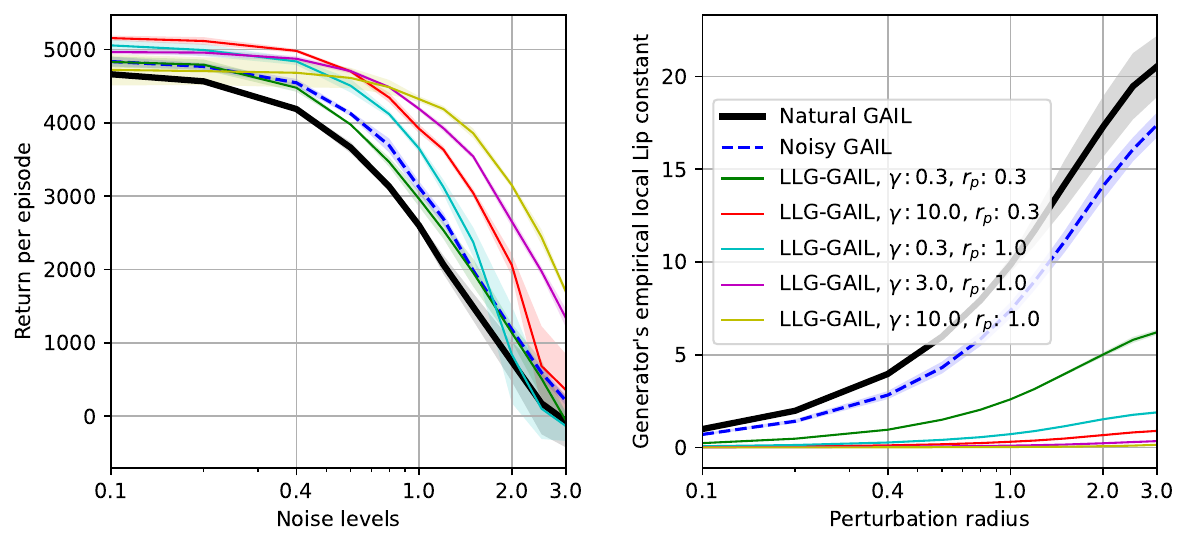}}%
\hfill
\subfigure[\footnotesize Ant, reg. disc.]{\label{subplot:Ant-d} \includegraphics[width=0.24\textwidth]{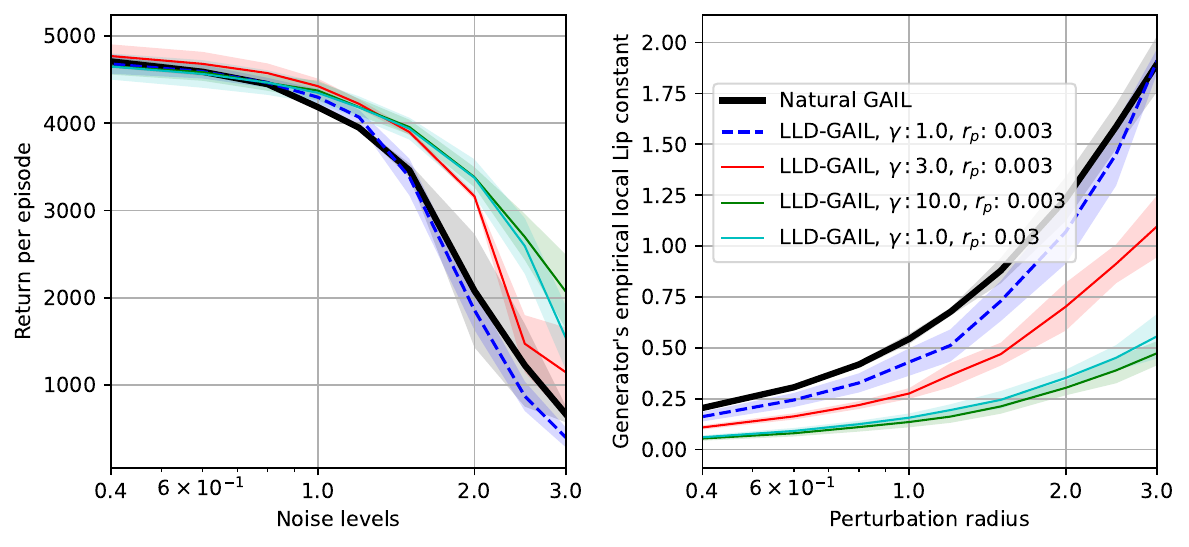}}%
\hfill
\subfigure[\footnotesize Ant, reg. gen.]{\label{subplot:Ant-g} \includegraphics[width=0.24\textwidth]{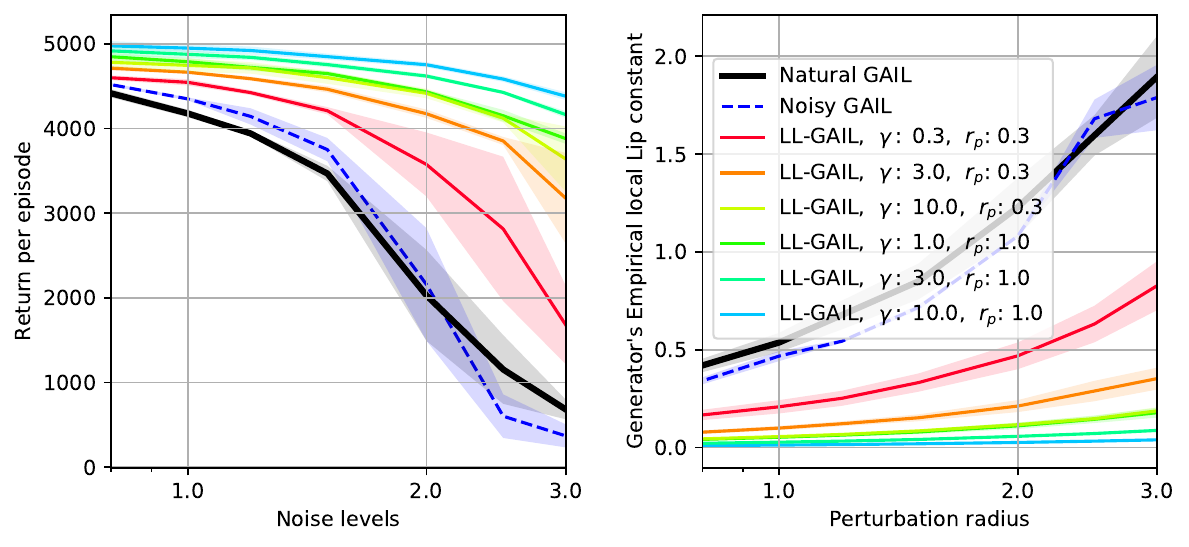}}%
    \caption{\footnotesize The comparison between LL-GAIL and the benchmarking schemes natural GAIL and noisy GAI on several simulated robot locomotion environments in the MuJoCo suite \cite{todorov2012mujoco}.
    The figures show the generators learned by LL-GAIL methods (either LLD-GAIL or LLG-GAIL) are more robust to observation noise compared to the baselines, as the proposed regularization methods improve the empirical local Lipschitzness constant (ELLC) of the trained generators.}
    \label{fig:dis}
    \vspace{-6mm}
\end{figure}

\vspace{-4mm}
\section{Experiments}

\label{sec:experiments}
\input{sections/sec_exp}

\vspace{-4mm}

\section{Conclusions}
\label{sec:conclusions}

We studied the robustness of GAIL to corrupted observations at test time. 
In such scenarios, the natural training of GAIL leads to learning policies that are highly sensitive to the level of observation noise. 
To remedy this shortcoming, we argued Lipschitz policies are more likely to remain agnostic to the observation noise. 
Subsequently, we proposed a regularization method to induce local Lipschitzness in the generator and the discriminator of adversarial imitation learning methods, which results in learning significantly more robust policies. We further provided theoretical insights and experimental support into the effectiveness of the proposed regularization method. 

\clearpage
\acks{This research was partly supported by the Army Research Lab and the National Science Foundation through the following grants: ARL W911NF2020132, NSF 1652113, and ARL ACC-APG-RTP W911NF1920333.

This work has also taken place partly in the Personal Autonomous Robotics Lab (PeARL) at The University of Texas at Austin. PeARL research is supported in part by the NSF (IIS-1724157, IIS-1638107, IIS-1749204, IIS-1925082), ONR (N00014-18-2243), AFOSR (FA9550-20-1-0077), and ARO (78372-CS).  This research was also sponsored by the Army Research Office under Cooperative Agreement Number W911NF-19-2-0333. The views and conclusions contained in this document are those of the authors and should not be interpreted as representing the official policies, either expressed or implied, of the Army Research Office or the U.S. Government. The U.S. Government is authorized to reproduce and distribute reprints for Government purposes notwithstanding any copyright notation herein.}

\bibliography{irl.bib}

\clearpage
\appendix
\section{Related Work}
\label{sec:related-work}

\input{sections/sec_related}

\input{sections/supp}

\end{document}

%% file: sections/sec_abstract.tex
We explore methodologies to improve the robustness of generative adversarial imitation learning (GAIL) algorithms to observation noise. 
Towards this objective, we study the effect of local Lipschitzness of the discriminator and the generator on the robustness of policies learned by GAIL.
In many robotics applications, the learned policies by GAIL typically suffer from a degraded performance at test time since the observations from the environment might be corrupted by noise.  Hence, robustifying the learned policies against the observation noise is of critical importance. To this end, we propose a regularization method to induce local Lipschitzness in the generator and the discriminator of adversarial imitation learning methods. We show that the modified objective leads to learning significantly more robust policies. 
Moreover, we demonstrate --- both theoretically and experimentally --- that training a locally Lipschitz discriminator leads to a locally Lipschitz generator, thereby improving the robustness of the resultant policy.
We perform extensive experiments on simulated robot locomotion environments from the MuJoCo suite that demonstrate the proposed method learns policies that significantly outperform the state-of-the-art generative adversarial imitation learning algorithm when applied to test scenarios with noise-corrupted observations.

%% file: sections/sec_intro.tex
Imitation learning enables the agents to learn directly from demonstrations and removes the burden of designing a utility function from system designers. Adversarial imitation learning (AIL) algorithms \cite{ho2016generative, fu2017learning} are a class of imitation learning algorithms that can learn an imitation policy in large environments with high-dimensional and continuous state and action spaces.  
Generative adversarial imitation learning (GAIL) \cite{ho2016generative} is one of the frequently used AIL algorithms. 

GAIL --- following generative adversarial networks (GAN) \cite{goodfellow2014generative} --- solves a min-max optimization problem between a discriminator and a generator. 
The discriminator is a classifier whose goal is to differentiate the state-action pairs produced by the generator from the demonstrations. 
The generator is a policy whose objective is to produce trajectories with similar state-action occupancy measures to the demonstrations. 
After successful training of GAIL, the generator can be utilized as a behavior policy for the autonomous agent in the environment.

Due to the instability in training GANs, and hence GAIL, gradient penalty is a common practice to introduce stability in the training process \cite{gulrajani2017improved}. Policies learned by GAIL \cite{ho2016generative} and its gradient penalized version performs well at test time if the test-time observations of the states are accurate. 
However, in certain scenarios, such as deploying an autonomous agent in an unknown, evolving environment, the observations at test time might be corrupted by noise due to factors such as sensor failure, evolving environmental conditions, and inconsistencies between the training and test environments \cite{brunke2022safe,zhao2023state}.
Under these scenarios, as we further demonstrate, GAIL (even with gradient penalty) leads to learning policies that are sensitive to noise at test time and perform poorly if the observations are corrupted by noise. Thus, ensuring the robustness and safety of decision-making methods have been an active area of research \cite{liu2022constrained,yu2022towards}.

It has been recently demonstrated that Lipschitzness improves the robustness of deep neural networks in classification tasks \cite{zhang2019theoretically, yang2020closer}. Inspired by the success of such Lipschitzness-inducing approaches, in this paper, we study the effect of promoting local Lipschitzness in GAIL-based methods for AIR. 

Intuitively, locally Lipschitz classifiers enjoy wider and smoother classification boundaries which in turn results in less sensitivity to inconsistencies between training and test data. 
The discriminator of GAIL is itself a classifier to differentiate the trajectories produced by the generator from the demonstrations. 
Furthermore, the robustness properties of the discriminator and the generator are critical in the robustness of the policy GAIL learns by solving the min-max optimization problem between the discriminator and the generator given that the generator, which is used as the policy in GAIL, is essentially a function of the discriminator.  Consequently, we argue while at test time we only use the generator for decision-making, the discriminator's properties such as its local Lipschitzness may affect the trained generator.
This discussion motivates us to investigate the effect of Lipschitz properties of the discriminator on the robustness of the learned generator to observation noise. 

To this end, we provide mathematical insights into the impact of the discriminator's local Lipschitzness on the robustness of the imitation policy. In particular, we show that under mild assumptions, the local Lipschitzness of the reward establishes the local Lipschitzness of the optimal action-value function. 
Since the discriminator can be thought of as a surrogate reward, we show inducing local Lipschitzness in the discriminator induces Lipschitzness in the learned policy. 
Hence, given that locally Lipschitz functions are agnostic to small variations in their input, training a locally Lipschitz discriminator then results in learning a policy that is robust to observation noise at test time.

Motivated by our theoretical insights, we propose methodologies to induce local Lipschitzness in GAIL to learn a robust policy. 
We first propose a novel regularized objective to train a locally Lipschitz discriminator. 
Once training is complete, we test the learned policy in test scenarios where the observations are corrupted by noise.

While the Lipschitzness of the discriminator may encourage robustness in the generator since they are trained jointly in a min-max game, we demonstrate enforcing local Lipschitzness on the generator directly is further beneficial. 
A locally Lipschitz generator is robust to the observation noise at test time. 
To induce local Lipschitzness in the generator, we propose a regularized objective that biases the generator towards locally Lipschitz solutions. 
By tuning the hyperparameters of the regularizer, we can control the induced radius of local Lipschitzness and the Lipschitz constant. 
Our extensive experimental results show the policies obtained with the regularized generator significantly outperform those obtained through naive GAIL in test scenarios where observations are corrupted by noise.

The summary of our contributions is as follows: 1) We study the effect of local Lipschitzness of the discriminator and the generator on the robustness of the learned policy through GAIL-based methods, 2) We propose a regularized objective to train the generator and the discriminator which induces local Lipschitzness on the learned policies, and 3) We provide thorough mathematical analysis to demonstrate how the local Lipschitzness properties of the discriminator translate into local Lipschitzness properties of the generator.

%% file: sections/sec_back.tex
\textbf{Reinforcement Learning (RL). }
A Markov decision process (MDP) is defined as a tuple $\mathcal{M} = \langle S, A,T,r , \gamma \rangle$ in which $S$ is the state space, $A$ is the action space, $T:S \times A \to \mathcal{P}(S)$ is the transition function which maps any state-action pair into a probability distribution over next states, $r(s,a):S\times A  \to \mathbb{R}$ is the reward function, and $\gamma \in (0,1)$ is the discount factor. 
A policy $\pi(a|s): S \to \mathcal{P}(A)$ is a probability distribution over actions at a state $s$. 
Given a policy $\pi$, we have the corresponding action-value function $Q^\pi(s,a)$, which is defined as:
\begin{align*}
    &Q^\pi(s,a) = r(s,a) + \mathbb{E}_{s'\sim T(s,a)} \mathbb{E}_{a' \sim \pi(a'|s')} [Q(s',a')] 
\end{align*}
Given an MDP $\mathcal{M}$, RL aims to find a policy with maximal expected discounted sum of future rewards. 
For a policy $\pi$, the discounted causal entropy is defined as \\$H(\pi) := \mathbb{E}_{(s,a) \in \rho_\pi} [-\log(\pi(a|s)/(1-\gamma)]$ in which $\rho_\pi$ is the state-action distribution induced by policy $\pi$.

\textbf{GAIL. }
Imitation learning algorithms \cite{ bain1995framework, ng2000algorithms, ziebart2008maximum, memarian2020active, brown2019extrapolating, brown2019ranking, memarian2021self} aim to learn a policy that mimics the underlying behavior of the demonstrations. 
Methods such as inverse reinforcement learning (IRL) \cite{ng2000algorithms, ziebart2008maximum, memarian2020active} do so by learning a reward function as an intermediate step. 
Solving an IRL problem involves repeatedly solving for a policy given the latest learned reward function, which makes IRL algorithms prohibitive to learn policies for large MDPs.
 GAIL was proposed \cite{ho2016generative} to overcome the above deficiency of IRL. 
Given a reward-free MDP $\mathcal{M} = \langle S, A, T , \gamma \rangle$ and an expert policy $\pi_E$, GAIL optimizes a regularized version of the IRL objective where the regularizer $\psi(c)$ is applied to the cost function $c$\footnote{The cost can be viewed as the negative of the reward.}
\begin{equation}
    \begin{aligned}
    \label{opt:cost-regularized-IRL-objective}
    \mathrm{IRL}_\psi&(\pi_E) = \argmax_{c \in \mathbb{R}^{S\times A}} \, \psi(c)& + \left(\min_{\pi \in \Pi} -H(\pi) + \mathbb{E}_\pi[c(s,a)] \right) - \mathbb{E}_{\pi_E}[c(s,a)].
    \end{aligned}
\end{equation}
Consider the following formulation for an entropy-regularized RL problem: \\$RL(c) = \argmin_\pi - H(\pi) + \mathbb{E}_\pi[c(s,a)].$
The original work proposing GAIL \cite{ho2016generative} proves that applying RL to a cost function learned through IRL is equivalent to
\begin{equation}
    \label{opt:GAIL-RL-o-IRL}
    \mathrm{RL} \left( \mathrm{IRL}_\psi(\pi_E) \right)= \argmin_{\pi} - H(\pi) + \psi^{*}(\rho_\pi - \rho_E)
\end{equation}
where $\psi^{*}$ is the convex conjugate of the regularizer $\psi$. 
By choosing a specific regularizer \cite{ho2016generative} reformulate the problem into
\begin{equation}
\begin{aligned}
    \label{opt:gail-objective}
    \argmin_\pi &\max_D ~ \mathbb{E}_{\pi}[\log(D(s,a))] &+ \mathbb{E}_{\pi_E}[\log(1-D(s,a))] - \lambda H(\pi)
    \end{aligned}
\end{equation}
where $D(s,a):S\times A \to (0,1)$ is a discriminative classifier. GAIL is typically trained with gradient penalty for stability; in that case a regularizer of the form $\frac{\kappa}{2}\|\nabla D(s,a)-1\|^2$ is added to \eqref{opt:gail-objective} where $\kappa$ is a hyper-parameter \cite{arjovsky2017wasserstein,orsini2021matters}. Because gradient penalty is essential for training stability of GAIL, we have incorporated it in all the baselines and the proposed method.
A detailed related work section is provided in the full version available at \cite{memarian2023robust}.

%% file: sections/sec_theorem.tex
We start by providing mathematical insight on how inducing Lipschitzness in the discriminator indirectly induces Lipschitzness in the generator, and in turn the imitation policy. 

When updating the parameters of the generator, the discriminator acts as a surrogate for the reward function, i.e., the generator update amounts to updating a policy through an RL algorithm using the discriminator to obtain the reward function. 
Hence, we alternatively study conditions on a discounted MDP $\mathcal{M} = \langle S, A,T, r , \gamma \rangle $ with stochastic dynamics and an $L$-Lipschitz reward function such that the corresponding optimal Q-function, $Q^{*}(s,a)$, becomes Lipschitz.
Next, we formalize locally Lipschitz functions.
\begin{definition} [\textbf{Locally Lipschitz function}]
    \label{def:local-lip-func}
    Consider the function $f(x): M_1 \to M_2$ which is a mapping from metric space $M_1$ to metric space $M_2$. 
    Let $d_{M_1}(.,.)$ and $d_{M_2}(.,.)$ be distance metrics defined on metric spaces $M_1$ and $M_2$ respectively. 
    Let $b_{M_1, r}(x_0) := \{ x \in M_1 \, | \, d_{M_1}(x,x_0)<r \}$ be the ball of radius $r$ around point $x$ defined by the metric $d_{M_1}$.
    Function $f(x)$ is L locally Lipschitz with radius $r$, if for every $x_0 \in M_1$ we have: $\forall x \in b_{M_1, r}(x_0): d_{M_2}(f(x),f(x_0)) < L \, d_{M_1}(x,x_0)$. If  $r = \infty$, we say that the function $f$ is $L$-Lipschitz.
\end{definition}

To show the Lipschitzness of the optimal Q-function, $Q^*$, we need to show that the norm of the gradient of $Q^*$ is bounded. 
We use the subscript $t$ to refer to the $t\ts{th}$ time step. 
Let $\nabla_{s_t} Q^*(s_t, a_t)$ denote the gradient of $Q^*$ at time step $t$ with respect to the state $s_t$, i.e., 
\begin{equation}
\label{eq:jacob-q-vector}
\begin{aligned}
\nabla_{s_t} Q^*(s_t, a_t)  = [ \nabla_{s_t^i} Q^*(s_t, a_t)  ]_{i=1}^{N}.
\end{aligned}
\end{equation}
where $\nabla_{s_t^i}$ is the gradient operator with respect to the $i^{th}$ dimension of the state space and $N$ is the dimension of the state space.
Then, it holds that
\begin{equation}
\label{eq:jacob-q2}
\begin{aligned}
\nabla_{s_t^i} Q^*(s_t, a_t)  &= \nabla_{s_t^i} \sum_{k=0}^\infty \gamma^k \, \mathbb{E}^*_{s_{t+k} |s_t,a_t} \left[ r(s_{t+k}) \right] &= \sum_{k=0}^\infty \gamma^k \, \nabla_{s_t^i}  \mathbb{E}^*_{s_{t+k} |s_t,a_t} \left[ r(s_{t+k}) \right],
\end{aligned}
\end{equation}
where $\mathbb{E}^{*}_{s_{t+k} |s_t,a_t}[.]$ denotes the expectation of its argument with respect to the conditional distribution of $s_{t+k}$ given that the agent starts from $s_t,a_t$ and follows the optimal policy. 

Theorem \ref{theorem:required-condition-stochastic-multi} below provides sufficient conditions under which the Frobenius norm of the gradient of the optimal Q-function is bounded. The proof is provided in the full version available at \cite{memarian2023robust}.
\begin{theorem}
    \label{theorem:required-condition-stochastic-multi}
    Consider an infinite-horizon discounted MDP $\mathcal{M} = \langle S, A,T, r , \gamma \rangle $ where the reward function is $L$-Lipschitz continuous.
    Let $\nabla_{s_{t}^i}$ be the gradient operator with respect to the $i^{th}$ dimension of the state space at time $t$.
    If there exists a constant $C$ such that the following inequalities hold for all $ i, k $:
    \begin{equation}
    \label{ineq:to-be-proved-inequality-theorem-multi}
    \begin{aligned}
        &\abs{ \nabla_{s_{t}^i}  \mathbb{E}^{*}_{s_{t+k} |s_{t}} \left[ r(s_{t+k}) \right] } 
        \leq  C&  \mathbb{E}^{*}_{s_{t+1} | s_{t}} \abs{  \nabla_{s_{t+1}^i} \mathbb{E}^{*}_{s_{t+k} | s_{t+1}} \left[  r(s_{t+k})   \right]},
    \end{aligned}
    \end{equation}
    then, it holds that $\norm{\nabla_{s_t} Q^*(s_t, a_t)}_F 
        \leq \sqrt{N} \, L \sum_{k=0}^\infty (\gamma \, C)^k.$
\end{theorem}
It is important to understand the implications of the inequality \eqref{ineq:to-be-proved-inequality-theorem-multi} of Theorem.~\ref{theorem:required-condition-stochastic-multi}. 
This inequality provides a sufficient condition for the Lipschitzness of the optimal Q-function with respect to the state dimensions. 
In simple words, inequality \eqref{ineq:to-be-proved-inequality-theorem-multi}
holds for a specific $i$, if the change in the expected value of the reward at time step $t+k$ is comparable for the following two cases: when we perturb the $i\ts{th}$ dimension of the state at time $t$ and when we perturb the $i\ts{th}$ dimension of the state at time $t+1$.
The optimal Q-function at any state-action pair is the expected value of the sum of the future discounted rewards when the agent follows the optimal policy starting from that state-action pair. 
Hence, for the gradient of the Q-function to be bounded, the gradient of the reward at a future time with respect to the current state needs to be bounded, and  \eqref{ineq:to-be-proved-inequality-theorem-multi} ensures that this condition is met. 
In the full version available at \cite{memarian2023robust}, we provide a simplified version of Theorem.~\ref{theorem:required-condition-stochastic-multi} for the special case of deterministic dynamics with a more explicit condition on the dynamics.

%% file: sections/sec_method.tex
As we show in Section \ref{sec:experiments}, the policies learned by GAIL are not robust to noise introduced at test time which may arise from deploying the policy in an evolving environment. 

As we established in Section \ref{sec:thm}, the Lipschitzness of the discriminator, under the conditions of Theorem \ref{theorem:required-condition-stochastic-multi}, implies the Lipschitzness of the optimal Q-function, and in turn the generator. Hence, in this section,
we propose a regularized version of GAIL, which we call locally-Lipschitz GAIL (LL-GAIL), that learns a robust policy through a local-Lipschitzness-inducing training procedure.\footnote{While we primarily focus on GAIL, the arguments may extend to other AIL algorithms as well.} 

Recent works study the link between accuracy and robustness in the context of deep neural network classifiers and conclude that inducing local Lipschitzness in the classifiers can enhance robustness without compromising accuracy \cite{zhang2019theoretically, yang2020closer}. Intuitively, locally Lipschitz classifiers enjoy wider and smoother classification boundaries which in turn results in less sensitivity to inconsistencies between training and test data. 
Inspired by these works, we study the effect of local Lipschitzness of the discriminator and the generator on the robustness of the imitation policy in generative adversarial imitation learning methods. 
We consider two scenarios: 1) We encourage local Lipschitzness in the discriminator through a novel regularization method and study how it affects the robustness of the resultant generator, and 2) We induce local Lipschitzness directly in the generator by regularizing the objective function of the generator and investigate the link between local Lipschitzness of the generator and its robustness to noise on observations at test time.

Similar to the original work on GAIL \cite{ho2016generative}, we use deep neural networks to represent the policy $\pi_\theta$ with parameters $\theta$ and the discriminator $D_\phi$ with parameters $\phi$. 
In the next two subsections, we discuss how we induce local Lipschitzness in the discriminator and the generator. 

\vspace{-3mm}
\subsection{Inducing Local Lipschitzness in the Discriminator}
\label{subsec:lip-disc}

To discuss our method to induce the local Lipschitzness of the discriminator, the first step is to define proper metrics for the input and output space of the discriminator. 
The discriminator $D:S \times A \to (0,1)$ is a classifier which maps the state-action space to a real number in the range $(0,1)$ specifying the probability that the state-action pair is sampled from the generator. 
In this work, we are only interested in robustness to noise on the observations, not the actions; hence, we only consider the local Lipschitzness properties of the discriminator with respect to the states and not the actions. This is mainly motivated by the fact that the acquired states at test time might be subject to noise due to changes in the environment or the failure of the agent's sensors.
The state space for the environments we are considering, i.e., robot locomotion environments, is a subspace of $\mathbb{R}^N$.
Consequently, two appropriate metrics for the state space are the $L_2$ norm and the $L_\infty$ norm. 
While we have performed experiments with both of these norms, we use the $L_2$ norm in the derivations presented in this paper.

We consider two options as the metric for the output space of the discriminator. 
One option is to use the raw output to construct a categorical probability distribution (with two classes) and use a metric such as the Jensen-Shannon divergence. 
The other option is to simply use the $L_1$ norm to measure the variations in the raw output of the discriminator which is a real number in the range $(0,1)$. 
We choose the second option in our derivation and the experiments.

In order to induce local Lipschitzness in the discriminator, we present a regularized version of the GAIL objective for updating the discriminator:
\begin{equation}
\begin{aligned}
    \label{opt:disc-reg}
    \argmax_{D} &~~ \mathbb{E}_{\pi_{\theta}}[\log(D(s,a))] &+ \mathbb{E}_{\pi_E}[\log(1-D(s,a))] - \gamma R_d(\mathcal{D}_d)
    \end{aligned}
\end{equation}
where $R_d(\mathcal{D}_d)$ is the regularization term and $\mathcal{D}_d$ is the training data for updating the discriminator at a given iteration. $\mathcal{D}_d$ consists of a collection of state-action pairs $(s,a)$ sampled from the generator, and a collection of state-action pairs $(s,a)$ sampled from the demonstrations. 


To compute the regularization term, for each $(s, a) \in \mathcal{D}_d$, we find an adversarial perturbation $\delta_{s,a}$ with an $L_2$ norm smaller than or equal to a hyper-parameter $r_p$:
\begin{equation}
\label{opt:adv-pert-disc}
\begin{aligned}
    \delta_{s,a} =
    \begin{cases}
    \argmax_\delta ~ \big|  D_\phi(s + \delta, a) - D_\phi(s,a) \big| \\
    s.t. ~~ ||\delta||_2 \leq r_p
    \end{cases}    
\end{aligned}
\end{equation}
and then we compute the regularization term as follows:
\begin{align*}
    R_d(\mathcal{D}_d) = \frac{1}{|\mathcal{D}_d|} \sum_{(s, a) \in \mathcal{D}_d} \big|  D_\phi(s+\delta_{s,a}, a) - D_\phi(s,a) \big|.
\end{align*}
The hyper-parameters $\gamma$ and $r_p$ aim to quantify the constant and the radius of the local Lipschitzness of the discriminator. 
The above procedure is summarized in Algorithm~\ref{alg:lip-disc}.

Since finding the optimal solution of \eqref{opt:adv-pert-disc} for each state-action pair is not computationally feasible, we instead propose to use projected gradient ascent with a pre-determined number of steps to get an approximate solution. 
Intuitively, the regularizer $R_d(\mathcal{D}_d)$ penalizes discriminators whose output varies significantly as a result of small perturbations in the input, and in turn induces local Lipschitzness. 
The above procedure is summarized as Algorithm~\ref{alg:lip-disc}.

\begin{algorithm}[t]
\caption{Updating the discriminator by regularizing for local Lipschitzness at iteration $i$}
\label{alg:lip-disc}
\begin{algorithmic}[1]
\STATE {\bfseries Input:} $D_{\phi_{i-1}}$: Current discriminator. $\pi_{\theta_{i-1}}$: Current generator.
\STATE {\bfseries Output:} An updated discriminator $D_{\phi_i}(s)$
\STATE {\bfseries Hyper-parameters:} $\gamma$: Regularization coefficient. $r_p$: perturbation radius. The choice between $L_2$ or $L_\infty$ norm to measure adversarial perturbations. 
    \STATE Form the training data $\mathcal{D}_d$ through a collection of state-action pairs sampled from generator $\pi_{\theta_{i-1}}$ and a collection of state-action pairs from the demonstrations
    \STATE Using $(s,a) \sim \mathcal{D}_d$, forward propagate through the discriminator to form GAIL's discriminator loss ($L_{ppo, d}$)
    \STATE For each $(s, a) \in \mathcal{D}_d$, perform $N$ steps of projected gradient ascent to find an adversarial perturbation $\delta_{s,a}$ within the $L_2$ (or $L_\infty$) ball of radius $r_p$:
    \begin{align*}
     \delta_{s,a} = 
        \begin{cases}
        \argmax_\delta ~ \abs{ (D_{\phi_{i-1}}(s, a)-  D_{\phi_{i-1}}(s+\delta, a))}  \\
         s.t. ~~ ||\delta||_2 \leq r_p ~~ \text{or} ~~ ||\delta||_\infty \leq r_p
        \end{cases}
    \end{align*}
    \STATE Forward propagates through the discriminator to form the regularization term as
        \begin{align*}
        R_d(\mathcal{D}_d) = \frac{\sum_{s, a \in \mathcal{D}_d} \abs{ D_{\phi_{i-1}}(s + \delta_{s,a}, a) - D_{\phi_{i-1}}(s, a) }}{|\mathcal{D}_d|} 
    \end{align*}

    \STATE Loss = $L_{ppo, d} + \gamma \times R_d(\mathcal{D}_d)$
    \STATE Back propagate through Loss to update the weights $\phi$ of the discriminator using an optimization algorithm of choice (Adam)
\end{algorithmic}
\end{algorithm}

\vspace{-3mm}
\subsection{Inducing Local Lipschitzness in the Generator}
\label{subsec:lip-gen}


The generator is a mapping from the state space to the space of probability distributions over the actions. 
As discussed in Section.~\ref{subsec:lip-disc}, we choose the $L_2$ norm as the metric on the state space. 
For the output space of the generator
we choose the Jeffreys divergence as a metric which is a symmetric version of the 
Kullback–Leibler (KL) divergence \cite{jeffreys1948theory}. In addition to the Jeffreys divergence, our framework can use the Jensen-Shannon divergence or any other metric over the space of probability distributions. 
The Jeffreys divergence between two probability distributions $p$ and $q$ is defined as 
\begin{equation}
    \label{def:jeffery-div}
    \begin{aligned}
    D_J(p \parallel q) &:= \int (p(x) - q(x))\big( \ln p(x) - \ln q(x) \big) dx 
    &= D_{KL}(p \parallel q) + D_{KL}(q \parallel p)
    \end{aligned}
\end{equation}
where $D_{KL}$ denotes the KL divergence. 

At a given iteration of the proposed LL-GAIL method, to collect the training data for the generator, we sample $m$ trajectories from the latest generator $\pi_{\theta}$ to form the set $\mathcal{D}_g = \{ \tau_j | \tau_j \sim \pi_{\theta} \}_{j=1}^m$. 
Our proposed regularized objective for generator updates is:
\begin{equation}
    \label{opt:gen-reg}
    \argmin_\pi ~  \mathbb{E}_{\pi}[\log(D(s,a))] - \lambda H(\pi) ~ +  ~ \gamma \, R_g(S_g),
\end{equation}
where $S_g = \{s | s \in \mathcal{D}_g \}$ is the set of all states in $\mathcal{D}_g$, and $R_g(S_g)$ is the regularization term.

We define $R_g(S_g)$ such that it encourages a locally Lipschitz generator. 
Intuitively, $R_g(S_g)$ penalizes generators whose output undergoes large variations as a result of small perturbations in their input. 
To compute $R_g(S_g)$, first, for every state $s \in S_g$ we compute $\delta_s$ by
\begin{equation}
\label{opt:adv-pert-gen}
    \delta_s = 
    \begin{cases} \argmax_\delta ~ D_J(\pi_\theta(s) \, || \, \pi_\theta(s+\delta)) \\
    s.t. ~~ ||\delta||_2 \leq r_p    
    \end{cases}
\end{equation}
where $r_p$ is a hyper-parameter that influences the radius of local Lipschitzness that the regularizer induces. 
We then compute $R_g(S_g)$ as follows:
\begin{align*}
    R_g(S_g) = \frac{1}{|S_g|} \sum_{s \in S_g} D_J(\pi_{\theta_i}(s) \, || \, \pi_{\theta_i}(s+\delta_s)).
\end{align*}
Note that $\delta_s$ is the perturbation within the $L_2$ ball of radius $r_p$ which causes the largest divergence in the policy. Since it is computationally infeasible to find the exact solution to \eqref{opt:adv-pert-gen} for every state, we instead use projected gradient ascent steps to get close to the solution. 

GAIL uses TRPO \cite{schulman2015trust} steps to update the generator. 
In this work, however, we use the simpler and more computationally efficient PPO algorithm \cite{schulman2017proximal} instead of TRPO for generator updates. 
Hence, we perform PPO steps on \eqref{opt:gen-reg} to update the generator. 
The above procedure is summarized as Algorithm~\ref{alg:lip-gen}.

\begin{algorithm}[t]
\caption{Updating Generator by Regularizing for Local Lipschitzness at iteration $i$}
\label{alg:lip-gen}
\begin{algorithmic}[1]
\STATE {\bfseries Input:} Current discriminator $D_{\phi_{i}}(s)$, current generator $\pi_{\theta_{i-1}}(a|s)$, 
\STATE {\bfseries Output:} An updated generator $\pi_{\theta_i}(a|s)$ 
\STATE {\bfseries Hyper-parameters:} Regularization coefficient $\gamma$. Perturbation radius $r_p$. The choice between $L_2$ or $L_\infty$ norm to measure adversarial perturbations. 
    \STATE Sample $m$ trajectories from the current generator $\pi_{\theta_{i-1}}(a|s)$ to form $\mathcal{D}_g$
    \STATE Use data points in $\mathcal{D}_g$ to forward propagate through the generator and form GAIL's generator loss $(L_{ppo, g})$
    \STATE Extract $S_g$ from $\mathcal{D}_g$
    \STATE For each $s \in S_g$ perform N steps of projected gradient ascent to find an adversarial perturbation $\delta_s$ within the $L_2$ (or $L_\infty$) ball of radius $r_p$:
    \begin{align*}
        \delta_s = 
        \begin{cases}
            \argmax_\delta ~ D_J(\pi_{\theta_{i-1}}(.|s) \, || \, \pi_{\theta_{i-1}}(.|s+\delta_s)) \\
            s.t. ~~ ||\delta||_2 \leq r_p ~~ \text{or} ~~ ||\delta||_\infty \leq r_p
        \end{cases}
    \end{align*}
    \STATE Forward propagates through the generator to form the regularization term
    \begin{align*}
        R_g(S_g) = \frac{1}{|\mathcal{D}_d|} \sum_{s \in S_g} J(\pi_{\theta_{i-1}}(.|s) \, || \, \pi_{\theta_{i-1}}(.|s+\delta_s))
    \end{align*}
    \STATE Loss $= L_{ppo,g} + \gamma \times R_g(\mathcal{D}_d)$
    \STATE Back propagate through Loss to update the weights $\theta$ of the generator using an optimization algorithm of choice (Adam)
\end{algorithmic}
\end{algorithm}

%% file: sections/sec_exp.tex
In this section, we demonstrate that policies learned by LL-GAIL are more robust to observation noise than those learned by the GAIL algorithm, which we call natural GAIL.\footnote{Additional experiments are in the full version available at \cite{memarian2023robust}.} 
In addition to natural GAIL, we benchmark our results against another baseline, which we call \textit{noisy GAIL}. 
The difference between noisy GAIL and natural GAIL is that noisy GAIL introduces random observation noise at training time to robustify the learned policy to observation noise at test time.

We perform experiments on several simulated robot locomotion environments in the MuJoCo suite \cite{todorov2012mujoco}, namely Walker2d, Hopper, HalfCheetah, and Ant.

\textbf{Regularizing the Discriminator.}
First, we investigate the effect of inducing local Lipschitzness in the discriminator of LL-GAIL
(see Figures~\ref{subplot:Walker2d-d}, \ref{subplot:Hopper-d},  \ref{subplot:HalfCheetah-d},\ref{subplot:Ant-d}). 
We refer to the discriminator-regularized LL-GAIL as LLD-GAIL.
Using optimization problem \ref{opt:disc-reg}, we train LLD-GAIL on a range of hyper-parameters $\gamma$ and $r_p$. 

Figure~\ref{subplot:Walker2d-d} benchmarks the LLD-GAIL against the baselines for Walker2d. 
The plot on the left-hand side of Figure~\ref{subplot:Walker2d-d} compares the performance of the generators learned by different models in test scenarios where the observations are corrupted by different levels of noise. 
To simulate the effect of noise on observations, we add zero-mean Gaussian noise to each dimension of the state space and the reported noise level is the standard deviation of the Gaussian noise. 

To understand the link between the local Lipschitzness of the generators and their robustness, the plot on the right-hand side of Figure~\ref{subplot:Walker2d-d} displays the empirical local Lipschitzness constant (ELLC) of the trained generators. 
The ELLC of a generator $\pi_\theta$ at a given radius $r_p$ is defined as $\mathbb{E}_{s, \norm{\delta}_2 = r_p} \left[ D_J(\pi_\theta(s)||\pi_\theta(s+\delta) / r_p \right]$, where the expectation is approximated by Monte Carlo sampling using 3840 samples which come from 30 trajectories of length 128 produced by the generator.
The ELLC is an empirical metric that quantifies the local Lipschitzness properties of a trained generator. 
A generator with a higher EELC is more sensitive to perturbations in the observations.

To train the noisy GAIL baseline in  Figures~\ref{subplot:Walker2d-d}, \ref{subplot:Hopper-d},  \ref{subplot:HalfCheetah-d},\ref{subplot:Ant-d}), we perturb the observations fed to the discriminator with zero-mean Gaussian noise. 
We train the noisy GAIL with a range of noise levels but we only report the results corresponding to the best training noise level (refer to the full version available at \cite{memarian2023robust} for more details). 

Figure~\ref{subplot:Walker2d-d} demonstrates that the proposed discriminator-regularization method helps LLD-GAIL outperform the baselines by improving the Lipschitzness properties of the generator and in turn those of the imitation policies. 
These observations corroborate the theoretical insights discussed in Section~\ref{sec:thm}. Figure~\ref{subplot:Walker2d-d} further shows that the generators that have a smaller ELLC at a given radius, perform better at a noise level comparable to that radius than generators that have a larger ELLC. The same pattern is observed in Figures.~\ref{subplot:Hopper-d}, \ref{subplot:HalfCheetah-d} when comparing the proposed method to the baselines in different environments.

\begin{wrapfigure}{r}{0.5\textwidth}
\vspace{-3mm}
  \begin{center}
    \includegraphics[width=0.48\textwidth]{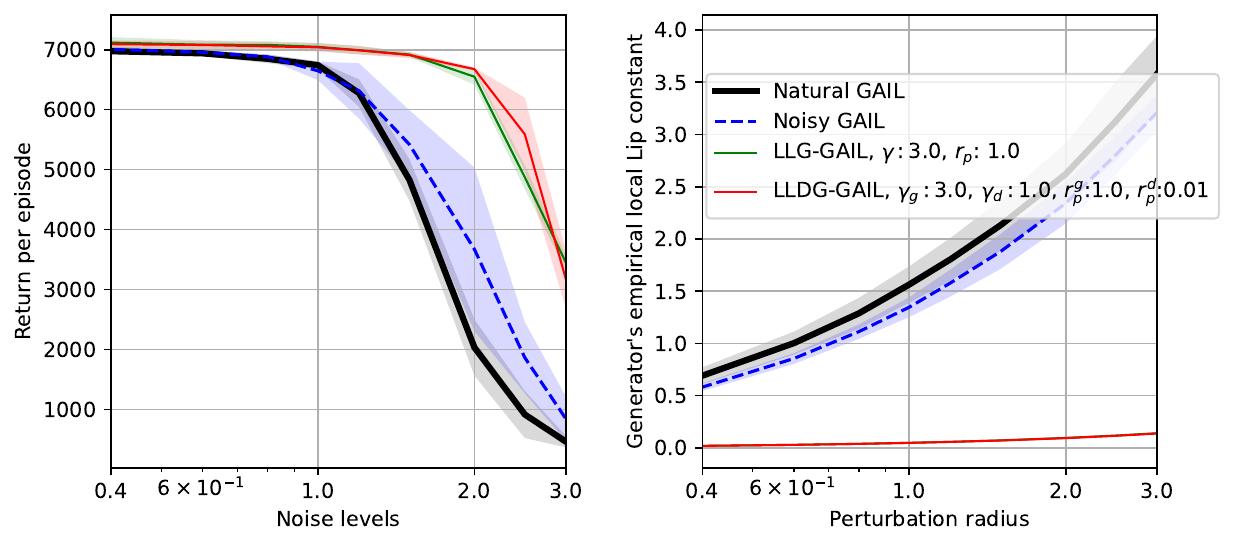}
  \end{center}
\caption{\footnotesize Walker2d experiment: LL-GAIL with both discriminator and generator regularizer outperforms all methods across various noise levels. $\gamma_g$ and $r^g_p$ are the hyper-parameters for regularizing the generator of LLDG-GAIL, and $\gamma_d$ and $r^d_p$ are the hyper-parameters for regularizing the discriminator of LLDG-GAIL.}
\label{plot:both-on}
\vspace{-6mm}
\end{wrapfigure}

\textbf{Regularizing the Generator.}
Next, we focus on the effect of directly inducing local Lipschitzness in the generator of LL-GAIL on the robustness of the generator (see Figures.~\ref{subplot:Walker2d-g}, \ref{subplot:Hopper-g},  \ref{subplot:HalfCheetah-g}, \ref{subplot:Ant-g}). 
We refer to this method as LLG-GAIL. 
We train the LLG-GAIL according to optimization problem \eqref{opt:gen-reg} (see the full version available at \cite{memarian2023robust}

When training the noisy GAIL baseline for this set of experiments, we add zero-mean Gaussian noise to the observations fed to the generator. 
We report the best results we obtained for the noisy GAIL baseline across different variances. 

Figure~\ref{subplot:Walker2d-g} depicts that LLG-GAIL significantly outperforms both natural GAIL and noisy GAIL across a wide range of noise levels. 
Comparing the left-hand plot and the right-hand plot of Figure~\ref{subplot:Walker2d-g}, we observe that the generators with smaller ELLCs vastly outperform those with larger ELLCs, especially at larger noise levels. 
This observation is consistent with our arguments in Sections \ref{sec:method} and \ref{sec:thm} about the vital role of local Lipschitzness of the generator in the robustness of the generator to observation noise. A similar pattern is observed in Figures.~\ref{subplot:Hopper-g}, \ref{subplot:HalfCheetah-g} when comparing the proposed method to the baselines in different environments. 

{\color{black}
\textbf{Regularizing both generator and discriminator.}
In this section, we investigate the benefit of regularizing both generator and discriminator. 
We refer to this method as LLDG-GAIL.
To this end, we compare LLDG-GAIL to the best of LLD-GAIL and LLG-GAIL. 
We do not perform an additional hyper-parameter tuning for LLDG-GAIL. Instead, for regularizing its discriminator, we use the best hyper-parameters we found for LLD-GAIL, and for regularizing its generator we use the best hyper-parameters we found for LLG-GAIL. 
As Figure \ref{plot:both-on} demonstrates, the LLDG-GAIL enjoys improved reward and ELLC compared to all of the benchmarking schemes across various noise levels for the Walker2d environment. For more experiments on LLDG-GAIL refer to the full version available at \cite{memarian2023robust}.
While LLDG-GAIL outperforms LLG-GAIL and LLD-GAIL, we observe diminishing returns compared to only regularizing the generator or discriminator. 
We identify two possible explanations for this phenomenon. 
First, both of the proposed regularizers aim to promote robustness through inducing local Lipschitzness, and given this common goal, observing a saturated improvement is expected. 
Additionally, the absence of hyperparameter tuning contributes to such a phenomenon.

}

%% file: sections/sec_related.tex
\textbf{Robust RL.} 
Recently, \cite{shen2020deep} applied a smoothness-inducing regularization to policies and Q-functions for both on-policy and off-policy RL methods and demonstrated improved sample efficiency and robustness. 
Different from \cite{shen2020deep}, we aim to improve the robustness of imitation learning algorithms as opposed to RL algorithms that learn from a pre-determined reward function. 
Moreover, we demonstrate -- both experimentally and theoretically -- that imposing local Lipschitzness on the discriminator, which acts as a surrogate for the reward function, leads to improving the robustness of the generator and in turn the learned policy. 
Reference \cite{pinto2017robust} proposes a method for robust adversarial RL by learning an additional adversarial policy. 
Their method makes the agent robust to adversarially perturbed environments by resorting to $H_\infty$
control methods. In contrast, we identify the local Lipschitzness of the discriminator and generator as two important factors for improved generalization and robustness of GAIL and propose methodologies to promote them.

\textbf{Robust Adversarial Imitation Learning.}
Adversarial inverse reinforcement learning (AIRL) \cite{fu2017learning} infers a reward function from demonstrations that is robust to changes in dynamics as the inferred reward is disentangled from the dynamics of the environment. 
This is different from our work since, we focus on the robustness of the learned policy, not the reward, and identify the local Lipschitzness of the discriminator as a mechanism to robustify the learned policy. 
Additionally, our method improves robustness with respect to noise on observations, whereas AIRL addresses robustness with respect to the dynamics of the MDP. 

\textbf{Regularized GANs.}
Spectral normalization \cite{miyato2018spectral}, weight clipping \cite{arjovsky2017wasserstein}, and gradient penalty \cite{gulrajani2017improved} are among different methods proposed recently to regularize the discriminator to improve the training stability of GANs. Divergent from these works, we are interested in improving the robustness of the policy learned by GAIL.
Recently \cite{choi2021variational} show that spectral clustering introduced in \cite{miyato2018spectral} improves the representation learning capabilities of generative models as it pertains to latent goal discovery in the context of goal-based RL. 
The Loss-Sensitive Generative Adversarial Network (LS-GAN) \cite{qi2020loss} induces a Lipschitz regularity condition on the density of real data, i.e., the space of distributions the GAN learns from, which leads to a regularized model that can generate more realistic samples than ordinary GANs. 
Conversely, we do not set any prior over our training data. Instead, we set a prior over the space of the functions to which the generator and discriminator belong. 

Another related work proposed the Wasserstein GAN (WGAN) algorithm \cite{arjovsky2017wasserstein}.
WGAN minimizes the Wasserstein distance between the data distribution and the generator's distribution.
Using the Kantorovich-Rubinstein duality, the objective is equivalent to a search over the space of k-Lipschitz discriminators. 
WGAN enforces Lipschitzness in the discriminator in a heuristic manner by clipping the weights of the corresponding function approximator which they admit is not the optimal way for enforcing Lipschitzness. 
Reference \cite{gulrajani2017improved}
impose 1-Lipschitzness in the discriminator by limiting the norm of the gradient of the discriminator to 1 at every state which leads to more stable training of WGANs. 

\textbf{Lipschitzness in Adversarial Imitation Learning. }
Wasserstein adversarial imitation learning  (WAIL) \cite{xiao2019wasserstein} extends WGANs to the space of adversarial imitation learning. 
WAIL casts the causal entropy regularized apprenticeship learning problem as minimizing the 1-Wasserstein distance between the occupancy measure of the policy and that of the expert.
By representing the Wasserstein distance in its dual form, the reward function appears as the Kantorovich potential and needs to be 1-Lipschitz. 
The reward function in WAIL is analogous to the discriminator in our formulation. 
A major difference between our method and WAIL is that instead of imposing 1-Lipschitzness of the reward function, we explore methodologies to promote local Lipschitzness of the generator and the discriminator to learn a robust policy. 

Recently, \cite{blonde2020lipschitzness} studied the effect of Lipschitzness of the discriminator on the performance of off-policy AIL methods. 
They use a gradient penalty regularizer to encourage the Lipschitzness of the discriminator. 
Their method improves the performance of GAIL in the training environment.  
However, differently from our work, the focus of \cite{blonde2020lipschitzness} is not on the robustness of the learned policy and they do not study the performance of the learned generator in the presence of observation noise. 
Additionally, while we identify the local Lipschitzness of the generator vital to the robustness of the learned policy, the study in  \cite{blonde2020lipschitzness} is limited to Lipschitzness of the discriminator and its effect on the return of the learned policy as opposed to its robustness. 
The concurrent work \cite{orsini2021matters} further shows that classical regularizers like dropout or weight decay perform on par with Lipschitzness-promoting methods in noiseless settings, while, as we argue, in noise-corrupted environments Lipschitzness plays a vital role in the robustness of the policies learned by GAIL.

%% file: sections/supp.tex
\section{Insights on Lipschitzness of the discriminator and the generator}
\label{app-subsec:theory}

As we discussed in Section \ref{sec:introduction} in the main text, the Lipschitzness properties of the discriminator are vital to the robustness of the trained generator. 
In this section, we provide mathematical insight into how inducing Lipschitzness in the discriminator indirectly induces Lipschitzness in the generator, and in turn the imitation policy. 

When updating the parameters of the generator, the discriminator acts as a surrogate for the reward function, i.e., the generator's update amounts to updating a policy through an RL algorithm using the discriminator to obtain a reward function. 
Hence, we alternatively study conditions on an infinite-horizon discounted MDP $\mathcal{M} = \langle S, A,T, r , \gamma \rangle $ with stochastic transition dynamics and an $L$-Lipschitz reward function such that the corresponding optimal Q-function, $Q^{*}(s,a)$, will be Lipschitz.
We recall the definition of a locally Lipschitz function formally below.
\begin{definition} [\textbf{Locally Lipschitz function}]
    Consider the function $f(x): M_1 \to M_2$ which is a mapping from metric space $M_1$ to metric space $M_2$. 
    Let $d_{M_1}(.,.)$ and $d_{M_2}(.,.)$ be distance metrics defined on metric spaces $M_1$ and $M_2$ respectively. 
    Let $b_{M_1, r}(x_0) := \{ x \in M_1 \, | \, d_{M_1}(x,x_0)<r \}$ be the ball of radius $r$ around point $x$ defined by the metric $d_{M_1}$.
    Function $f(x)$ is L locally Lipschitz with radius $r$, if for every $x_0 \in M_1$ we have: $\forall y \in b_{M_1, r}(x_0): d_{M_2}(f(x),f(x_0)) < L \, d_{M_1}(x,x_0)$. If  $r = \infty$, we say that the function $f$ is $L$-Lipschitz.
\end{definition}

We model the environment with an MDP $\mathcal{M} = \langle S, A,T, r , \gamma \rangle $ where the reward function is $L$-Lipschitz continuous.
We further assume that the task of interest is an infinite-horizon continuing task and use the subscript $t$ to refer to the $t\ts{th}$ time step. 
In this work, we are interested in the Lipschitzness of the reward function and the optimal Q-function with respect to the states and not the actions. 
We assume the state space is continuous and the reward function is differentiable everywhere.
We consider the reward function to be a function of states only, i.e., $r(s) :S \to A$.  
With these assumptions, the $L$-Lipschitzness of the reward function translates to $ \norm{\nabla_{s} r(s)}_p < L, \forall s \in S $ for some $L_p$ norm.
Moreover, we define the \textit{optimal Markov chain} as the Markov chain that is obtained by applying the greedy optimal policy to the MDP. 

First, we focus on the case where the state space is one dimensional and then we extend the analysis to the general case of multi-dimensional state space. 

\subsection{Lipschitzness of Optimal Q-function for One-Dimensional State Space}
\label{app-subsec:theory-one-dim}

In this subsection, we investigate the required conditions for the Lipschitz continuity of the optimal Q-function given the Lipschitz continuity of the reward function for the case where the state space is one dimensional. 
The analysis extends straightforwardly to multi-dimensional state-spaces which we cover in the next subsection.

To show the Lipschitz continuity of the optimal Q-function, we need to show that the magnitude of the gradient of the optimal Q-function with respect to the states is bounded.
For a given non-negative integer $k$, we use the notation $\mathbb{E}^{*}_{s_{t+k} |s_t,a_t}[.]$ to denote the expectation of the argument with respect to the conditional distribution of $s_{t+k}$ given that the agent starts from $s_t,a_t$ and follows the optimal policy. 
Similarly, we use the notation $\mathbb{E}^{*}_{s_{t+k} |s_t}[.]$ to denote the expectation of the argument with respect to the conditional distribution of $s_{t+k}$ given that the agent starts from $s_t$ and follows the optimal policy. 

The optimal Q-function, $Q^*$, at $(s_t, a_t)$ can be written as:
\begin{equation} \label{eq:opt-q}
Q^*(s_t, a_t) = \sum_{k=0}^{\infty} \gamma^k \, \mathbb{E}^{*}_{s_{t+k} |s_t,a_t} \left[ r(s_{t+k}) \right].
\end{equation}
Hence the gradient of the optimal Q-function with respect to the state at time $t$ is
\begin{equation} \label{eq:grad-q}
\begin{aligned}
\nabla_{s_t} Q^*(s_t, a_t) &= \nabla_{s_t} \sum_{k=0}^{\infty} \gamma^k \, \mathbb{E}^{*}_{s_{t+k} |s_t,a_t} \left[ r(s_{t+k}) \right] \\
&= \sum_{k=0}^{\infty} \gamma^k \, \nabla_{s_t}  \mathbb{E}^{*}_{s_{t+k} |s_t,a_t} \left[ r(s_{t+k}) \right].
\end{aligned}
\end{equation}
The individual terms of the above sum, i.e., $\gamma^k  \nabla_{s_t}  \mathbb{E}^{*}_{s_{t+k} |s_t,a_t} \left[ r(s_{t+k}) \right]$, measure the change in the expectation of the discounted reward at a future time step, $t+k$, given an infinitesimal perturbation at the current state $s_t$.
In order to upper bound the magnitude of the gradient of $Q^{*}$, 
we need to upper bound the individual terms involved in \eqref{eq:grad-q}, i.e., $\gamma^k \nabla_{s_t}  \mathbb{E}^{*}_{s_{t+k} |s_t,a_t} \left[ r(s_{t+k}) \right]$, and show that their sum is  bounded. 
To this end, we state and prove the following proposition as an intermediate step.
\begin{proposition} \label{prop:first-prop}
    If the following inequality holds for all $t$ and all $k \in \{0, 1, \cdots\}$ 
    \begin{equation} \label{ineq:required-inequality-2}
    \begin{aligned}
        \abs{ \nabla_{s_{t}}  \mathbb{E}^{*}_{s_{t+k} |s_{t}} \left[ r(s_{t+k}) \right] } 
        &\leq  C\,  \mathbb{E}^{*}_{s_{t+1} | s_{t}} \left[ \abs{  \nabla_{s_{t+1}} \mathbb{E}^{*}_{s_{t+k} | s_{t+1}} \left[  r(s_{t+k})   \right]}  \right],
    \end{aligned}
    \end{equation}
    then it holds for every $k \in \{0, 1, \cdots\}$ that:
    \begin{equation}
    \begin{aligned}
        \abs{   \nabla_{s_t}  \mathbb{E}^{*}_{s_{t+k} |s_t,a_t} \left[ r(s_{t+k}) \right] } \leq C^k  L
    \end{aligned}
    \end{equation}    
\end{proposition}
\begin{proof}
Starting from inequality \eqref{ineq:required-inequality-2}, we have:
\begin{align*}
    \abs{ \nabla_{s_t}  \mathbb{E}^{*}_{s_{t+k} |s_t,a_t} \left[ r(s_{t+k}) \right] } 
    & \leq  C\,  \mathbb{E}^{*}_{s_{t+1} | s_t} \left[ \abs{  \nabla_{s_{t+1}} \mathbb{E}^{*}_{s_{t+k} | s_{t+1}} \left[  r(s_{t+k})   \right]} \right] \\
    & \leq  C\,  \mathbb{E}^{*}_{s_{t+1} | s_t} C \mathbb{E}^{*}_{s_{t+2} | s_{t+1}} \left[ \abs{  \nabla_{s_{t+2}} \mathbb{E}^{*}_{s_{t+k} | s_{t+2}} \left[  r(s_{t+k})   \right]} \right] ~~~~~ \tag{From \eqref{ineq:required-inequality-2}} \\
    & = C^2 \,  \mathbb{E}^{*}_{s_{t+1} | s_t} \mathbb{E}^{*}_{s_{t+2} | s_{t+1}} \left[ \abs{  \nabla_{s_{t+2}} \mathbb{E}^{*}_{s_{t+k} | s_{t+2}} \left[  r(s_{t+k})   \right]} \right] \\
    & = C^2 \,  \mathbb{E}^{*}_{s_{t+2} | s_t} \left[ \abs{  \nabla_{s_{t+2}} \mathbb{E}^{*}_{s_{t+k} | s_{t+2}} \left[  r(s_{t+k})   \right]}  \right] \\
    & \vdots ~~~~~ \tag{Repeated use of \eqref{ineq:required-inequality-2}} \\
    &\leq C^k \,  \mathbb{E}^{*}_{s_{t+k} | s_t} \left[ \abs{  \nabla_{s_{t+k}} \mathbb{E}^{*}_{s_{t+k} | s_{t+k}} \left[  r(s_{t+k})   \right]} \right] \\
    &= C^k \,  \mathbb{E}^{*}_{s_{t+k} | s_t} \left[ \abs{  \nabla_{s_{t+k}} \left[  r(s_{t+k})   \right]} \right] \\
    &\leq C^k \,  \mathbb{E}^{*}_{s_{t+k} | s_t} [L] \tag{Due to Lipschitzness of reward} \\ 
    &= C^k \, L 
\end{align*}
\end{proof}

The following theorem provides a sufficient condition for upper bounding the magnitude of the gradient of the $Q^{*}$ which is required to prove the Lipschitzness of $Q^{*}$.

\begin{theorem} \label{theorem:first-theorem}
    If the following inequality holds for all $t$ and all $k \in \{0, 1, \cdots\}$ 
\begin{align} \label{ineq:required-inequality-2-theorem}
        \abs{ \nabla_{s_{t}}  \mathbb{E}^{*}_{s_{t+k} |s_{t}} \left[ r(s_{t+k}) \right] } 
        &\leq  C\,  \mathbb{E}^{*}_{s_{t+1} | s_{t}} \abs{  \nabla_{s_{t+1}} \mathbb{E}^{*}_{s_{t+k} | s_{t+1}} \left[  r(s_{t+k})   \right]},
\end{align}
then it holds for all $t$ that
\begin{equation}
\label{eq:norm-grad-q-upper-bound-by-c}
\abs{\nabla_{s_t} Q^*(s_t, a_t)} 
\leq L \sum_{k=0}^\infty (\gamma \, C)^k.
\end{equation}
\end{theorem}
\begin{proof}

\begin{align*} \label{eq:norm-grad-q}
    \abs{\nabla_{s_t} Q^*(s_t, a_t)} 
    &= \abs{ \sum_{k=0}^\infty \gamma^k \, \nabla_{s_t}  \mathbb{E}^{*}_{s_{t+k} |s_t,a_t} \left[ r(s_{t+k}) \right] } \\
    &\leq \sum_{k=0}^\infty \gamma^k \abs{  \nabla_{s_t}  \mathbb{E}^{*}_{s_{t+k} |s_t,a_t} \left[ r(s_{t+k}) \right] }  \tag{From.~\eqref{ineq:required-inequality-2-theorem}} \\
    &\leq  \sum_{k=0}^\infty \gamma^k  C^k . L  \tag{From proposition.~\ref{prop:first-prop}}\\
    &= L \sum_{k=0}^\infty (\gamma \, C)^k.
\end{align*}

\end{proof}

The above theorem states a sufficient condition for the magnitude of the gradient of the optimal Q-function to be upper-bounded. 
Note that $\sum_{k=0}^\infty (\gamma \, C)^k $ is a geometric series which is finite and equal to $\frac{1}{1 - \gamma C}$ only if $\gamma \, C < 1$. 
Since $\gamma < 1$ by definition, all we need is $C \leq 1$.

It is important to understand the intuition behind \eqref{ineq:required-inequality-2-theorem}.
Let $f_k(s_{t+j}) := \mathbb{E}^{*}_{s_{t+k} |s_{t+j}}  \left[ r(s_{t+k}) \right]$ for any $j \in \{0,1, \cdots, k \}$.

For the left hand side of \eqref{ineq:required-inequality-2-theorem} we have:
\begin{align}
    \label{lhs-introduce-f}
    \abs{ \nabla_{s_{t}}  \mathbb{E}^{*}_{s_{t+k} |s_{t}} \left[ r(s_{t+k}) \right] } 
    &= \abs{ \nabla_{s_{t}}   \mathbb{E}^{*}_{s_{t+1} |s_{t}} \underbrace{\mathbb{E}^{*}_{s_{t+k} |s_{t+1}}  \left[ r(s_{t+k}) \right]}_{f_k(s_{t+1})} } \nonumber \\
    &= \abs{ \nabla_{s_{t}} \mathbb{E}^{*}_{s_{t+1} |s_{t}}  \left[ f_k(s_{t+1}) \right]},
\end{align}
and for the right hand side of \eqref{ineq:required-inequality-2-theorem} we have:
\begin{align}
    \label{rhs-introduce-f}
    C\,  \mathbb{E}^{*}_{s_{t+1} | s_{t}} \abs{  \nabla_{s_{t+1}} \mathbb{E}^{*}_{s_{t+k} | s_{t+1}} \left[  r(s_{t+k})   \right]}
    &=  C\,  \mathbb{E}^{*}_{s_{t+1} | s_{t}} \abs{  \nabla_{s_{t+1}} f_k(s_{t+1})}
\end{align}
Hence, using \eqref{lhs-introduce-f} and \eqref{rhs-introduce-f}, inequality \eqref{ineq:required-inequality-2-theorem} reduces to:
\begin{align} \label{ineq:simplified-condition-stoch}
    \abs{ \nabla_{s_{t}} \mathbb{E}^{*}_{s_{t+1} |s_{t}}  \left[ f_k(s_{t+1}) \right]} \leq C\,  \mathbb{E}^{*}_{s_{t+1} | s_{t}} \abs{  \nabla_{s_{t+1}} f_k(s_{t+1})}
\end{align}

To understand this inequality better, we first consider the special case where the transition dynamics of the MDP are deterministic.
In that case, assuming a greedy optimal policy, the transition dynamics of the optimal Markov chain will be deterministic as well, i.e., $s_{t+1} = D(s_t)$ where $D(s): S \to S$ is the function specifying the deterministic dynamics of the optimal Markov Chain. 
For the left hand side of \eqref{ineq:simplified-condition-stoch} we have:
\begin{align*} 
    \abs{ \nabla_{s_{t}} \mathbb{E}^{*}_{s_{t+1} |s_{t}}  \left[ f_k(s_{t+1}) \right]} 
    &= \abs{ \nabla_{s_{t}}   f_k(s_{t+1}) } \tag{Due to deterministic transitions}  \\
    &= \abs{ \nabla_{s_{t}}s_{t+1} \times  \nabla_{s_{t+1}}  f_k(s_{t+1}) } \tag{Chain rule} \\
    &= \abs{ \nabla_{s_{t}}D(s_t) \times  \nabla_{D(s_t)}  f_k(D(s_t)) } \tag{Definition of $D(s_t)$} 
\end{align*}
For the right hand side of \eqref{ineq:simplified-condition-stoch} we have:
\begin{align*} 
    C \, \mathbb{E}^{*}_{s_{t+1} | s_{t}} \abs{  \nabla_{s_{t+1}} f_k(s_{t+1})}
    &= C \, \abs{  \nabla_{D(s_t)} f_k(D(s_t))} \tag{Definition of $D(s_t)$}
\end{align*}
Putting the left-hand side and the right-hand side together for the case of the deterministic dynamic, we get;
\begin{align} \label{ineq:simplified-condition-det}
    \abs{ \nabla_{s_{t}}D(s_t) \times \nabla_{D(s_t)}  f_k(D(s_t)) } \leq C \, \abs{  \nabla_{D(s_t)} f_k(D(s_t))} 
\end{align}
And for the above inequality to hold, it is sufficient to have $\abs{ \nabla_{s_t} D(s_t) } \leq C$, which  holds if the dynamics of the optimal Markov chain is $C$-Lipschitz. 

Essentially, inequality \eqref{ineq:required-inequality-2-theorem} and it's equivalent form \eqref{ineq:simplified-condition-stoch}, are a version of 
\eqref{ineq:simplified-condition-det} where the transition dynamics are stochastic rather than deterministic.
In light of the insights from \eqref{ineq:simplified-condition-det}, inequalities \eqref{ineq:required-inequality-2-theorem} and \eqref{ineq:simplified-condition-stoch} will hold if the stochastic dynamics of the MDP and the corresponding stochastic dynamics of the optimal Markov chain have a property that resembles Lipschitzness of a deterministic function.

\subsection{Lipschitzness of Optimal Q-function, Multi-Dimensional State Space}
\label{app-subsec:theory-multi}

The analysis in the previous subsection straightforwardly extends to multi-dimensional state space, but for the sake of completeness, we dedicate this subsection to the case of multi-dimensional state space.

For the case of multi-dimensional state space, we define $\nabla_{s_t^i}$ as the gradient operator with respect to the $i^{th}$ dimension of the state space. 
While we can use any $L_p$ norm in our derivations to measure variations in the state space, we choose to write the derivations based on $L_2$ norm.
We assume the reward function is $L$-Lipschitz continuous which mean $\norm{\nabla_{s_t} r(s_t) }_2 < L, \forall s\in S$. 
From the $L$-Lipschitzness of the reward function we can conclude that $\norm{\nabla_{s_{t}^i} r(s_t) }_2 < L, \forall s\in S$.

To show the Lipschitz continuity of the optimal Q-function for the general case of multi-dimensional state space, we need to show that the $L_2$ norm of the gradient of the optimal Q-function is bounded. 
Recall $\nabla_{s_t} Q^*(s_t, a_t)$ denotes the gradient of the optimal Q-function at time step $t$ with respect to the state $s_t$, i.e.,
\begin{equation*}
\begin{aligned}
\nabla_{s_t} Q^*(s_t, a_t)  = [ \nabla_{s_t^i} Q^*(s_t, a_t)  ]_{i=1}^{N}.
\end{aligned}
\end{equation*}
where  $N$ is the dimensionality of the state space.
Then, it holds that
\begin{equation}
\label{eq:jacob-q2-supp}
\begin{aligned}
\nabla_{s_t^i} Q^*(s_t, a_t)  &= \nabla_{s_t^i} \sum_{k=0}^\infty \gamma^k \, \mathbb{E}^*_{s_{t+k} |s_t,a_t} \left[ r(s_{t+k}) \right] = \sum_{k=0}^\infty \gamma^k \, \nabla_{s_t^i}  \mathbb{E}^*_{s_{t+k} |s_t,a_t} \left[ r(s_{t+k}) \right],
\end{aligned}
\end{equation}

The individual terms of the above sum, i.e., $\gamma^k \nabla_{s_t^i}  \mathbb{E}^{*}_{s_{t+k} |s_t,a_t} \left[ r(s_{t+k}) \right]$, measure the change in the expectation of the discounted reward at a future time step, $t+k$, given an infinitesimal perturbation in the $i^{th}$ dimension of the current state $s_t$.
In order to upper bound the norm of the gradient of $Q^{*}$, 
we need to upper bound the individual terms involved in \eqref{eq:jacob-q2-supp}, i.e., $\gamma^k  \nabla_{s_t^i}  \mathbb{E}^{*}_{s_{t+k} |s_t,a_t} \left[ r(s_{t+k}) \right]$, and show that their sum is  bounded. 
To this end, we state and prove the following proposition as an intermediate step.

\begin{proposition} \label{prop:second-inequality-multi}
    For any $i \in \{ 1, 2, \cdots, N\}$, if the following inequalities hold for all $t$ and all $k \in \{0,1,\cdots \}$,
    \begin{equation}
    \label{ineq:required-inequality-2-multi}
    \begin{aligned}
        \abs{ \nabla_{s_{t}^i}  \mathbb{E}^{*}_{s_{t+k} |s_{t}} \left[ r(s_{t+k}) \right] } 
        &\leq  C\,  \mathbb{E}^{*}_{s_{t+1} | s_{t}} \abs{  \nabla_{s_{t+1}^i} \mathbb{E}^{*}_{s_{t+k} | s_{t+1}} \left[  r(s_{t+k})   \right]},
    \end{aligned}
    \end{equation}
    then it holds that:
    \begin{equation} 
    \begin{aligned}
        \forall k, ~ \abs{   \nabla_{s_t^i}  \mathbb{E}^{*}_{s_{t+k} |s_t,a_t} \left[ r(s_{t+k}) \right] } \leq C^k . L,
    \end{aligned}
    \end{equation}    
\end{proposition}

\begin{proof}
Starting from \eqref{ineq:required-inequality-2-multi}, we have:
\begin{align*} \label{ineq:required-inequality-2-expanded-multi}
    \abs{ \nabla_{s_t^i}  \mathbb{E}^{*}_{s_{t+k} |s_t,a_t} \left[ r(s_{t+k}) \right] } 
    & \leq  C\,  \mathbb{E}^{*}_{s_{t+1} | s_t} \abs{  \nabla_{s_{t+1}^i} \mathbb{E}^{*}_{s_{t+k} | s_{t+1}} \left[  r(s_{t+k})   \right]}   \\
    & \leq  C\,  \mathbb{E}^{*}_{s_{t+1} | s_t} C \mathbb{E}^{*}_{s_{t+2} | s_{t+1}} \abs{  \nabla_{s_{t+2}^i} \mathbb{E}^{*}_{s_{t+k} | s_{t+2}} \left[  r(s_{t+k})   \right]}  \tag{{From \eqref{ineq:required-inequality-2-multi}}} \\
    & = C^2 \,  \mathbb{E}^{*}_{s_{t+1} | s_t} \mathbb{E}^{*}_{s_{t+2} | s_{t+1}} \abs{  \nabla_{s_{t+2}^i} \mathbb{E}^{*}_{s_{t+k} | s_{t+2}} \left[  r(s_{t+k})   \right]} \\
    & = C^2 \,  \mathbb{E}^{*}_{s_{t+2} | s_t} \abs{  \nabla_{s_{t+2}^i} \mathbb{E}^{*}_{s_{t+k} | s_{t+2}} \left[  r(s_{t+k})   \right]} \\
    & \vdots  \tag{Repeated use of \eqref{ineq:required-inequality-2-multi}} \\
    &\leq C^k \,  \mathbb{E}^{*}_{s_{t+k} | s_t} \abs{  \nabla_{s_{t+k}^i} \mathbb{E}^{*}_{s_{t+k} | s_{t+k}} \left[  r(s_{t+k})   \right]} \\
    &\leq C^k \,  \mathbb{E}^{*}_{s_{t+k} | s_t} \abs{  \nabla_{s_{t+k}^i} \left[  r(s_{t+k})   \right]}
\end{align*}
but due to the $L$-Lipschitznes of the reward function, we have $\mathbb{E}^{*}_{s_{t+k} | s_t} \abs{  \nabla_{s_{t+k}^i} \left[  r(s_{t+k})   \right]} \leq L$. 
Consequently, we conclude
\begin{equation}
    \abs{ \nabla_{s_t^i}  \mathbb{E}^{*}_{s_{t+k} |s_t,a_t} \left[ r(s_{t+k}) \right] } \leq C^k \, L.
\end{equation}

\end{proof}

The following theorem provides a sufficient condition for upper bounding the norm of the gradient of the $Q^{*}$ which is required to prove the Lipschitzness of $Q^{*}$.

\begin{theorem} \label{theorem:multi}
If there exist constant $C$ where for each $i \in \{1, \cdots, N\}$
    \begin{equation} \label{ineq:required-inequality-2-multi-theorem}
    \begin{aligned}
        \abs{ \nabla_{s_{t}^i}  \mathbb{E}^{*}_{s_{t+k} |s_{t}} \left[ r(s_{t+k}) \right] } 
        &\leq  C\,  \mathbb{E}^{*}_{s_{t+1} | s_{t}} \abs{  \nabla_{s_{t+1}^i} \mathbb{E}^{*}_{s_{t+k} | s_{t+1}} \left[  r(s_{t+k})   \right]},
    \end{aligned}
    \end{equation}
then we can upper bound the $L_2$ norm of the gradient of the Q-function as follows:
\begin{equation}
\label{eq:norm-grad-q-upper-bound-by-c-multi}
\norm{\nabla_{s_t} Q^*(s_t, a_t)}_2 
\leq \sqrt{N} \,  L \sum_{k=0}^\infty (\gamma \, C)^k
\end{equation}
\end{theorem}
\begin{proof}

    \begin{align*} \label{eq:norm-jacob-q}
    \norm{\nabla_{s_t} Q^*(s_t, a_t)}_2^2 &= 
    \sum_i^N \left( \nabla_{s_t^i} Q^*(s_t, a_t) \right)^2 \\
    &= \sum_i^N \left( \sum_{k=0}^\infty \gamma^k \, \nabla_{s_t^i}  \mathbb{E}^{*}_{s_{t+k} |s_t,a_t} \left[ r(s_{t+k}) \right] \right)^2 \\
    &\leq \sum_i^N \left( \sum_{k=0}^\infty \gamma^k \abs{  \nabla_{s_t^i}  \mathbb{E}^{*}_{s_{t+k} |s_t,a_t} \left[ r(s_{t+k}) \right] } \right)^2 \\
    &\leq \sum_i^N \left( \sum_{k=0}^\infty \gamma^k C^k . L \right)^2 ~~~~~~~ \tag*{ Using \eqref{ineq:required-inequality-2-multi-theorem}} \\
    &\leq \sum_i^N \left( \sum_{k=0}^\infty (\gamma C)^k . L \right)^2 \\
    &=  N \left( L \sum_{k=0}^\infty (\gamma C)^k  \right)^2 
    \end{align*}
    By taking the square root of both sides, we get:
    \begin{align*}
        \norm{\nabla_{s_t} Q^*(s_t, a_t)}_2 
        \leq \sqrt{N} \,  L \sum_{k=0}^\infty (\gamma \, C)^k
    \end{align*}

\end{proof}

The above theorem states a sufficient condition for the magnitude of the gradient of the optimal Q-function with respect to the states to be upper bounded. 
Note that $\sum_{k=0}^\infty (\gamma \, C)^k $ is a geometric series which is finite and equal to $\frac{1}{1 - \gamma C}$ only if $\gamma \, C < 1$. 
Since $\gamma < 1$ by definition, all we need is $C \leq 1$.

For the intuitions behind the condition \eqref{ineq:required-inequality-2-multi-theorem} of the above Theorem, refer to the previous subsection on one-dimensional state spaces. 
The same argument extends to multi-dimensional state spaces.

\clearpage
\section{Additional Experiments}
\label{app-subsec:experiments}
Here is a summary of some of the most important hyper-parameters we use in training LL-GAIL: Learning rate for discriminator and generator updates: $3.0e-4$. Discount factor for reward: $\gamma=0.99$. Total number of environment steps for training the generator through PPO: $15,000,000$. Number of PPO epochs: $10$. We use linear decay for training the PPO algorithm. Entropy coefficient of $0$ for PPO algorithm. Number of projected gradient ascent steps to compute $\delta_s$ and $\delta_{s,a}$ through \eqref{opt:adv-pert-disc}, \eqref{opt:adv-pert-gen} in the main text: 10 steps. PPO clipping parameter: $0.2$. We apply gradient penalty to all methods and tune their parameters by performing a grid search. For all cases of gradient penalty, we report the results for a $\kappa$ value of 10.

Except for the hyper-parameters that are specific to LL-GAIL, the two baselines use identical hyper-parameters as LL-GAIL.

The neural network modeling the discriminator is a fully connected network. 
It takes as input the observation and action and outputs a real number in the range $(0,1)$. 
The network consists of two hidden layers of size 100 each followed by a $tanh$ layer. 
The output layer is one-dimensional and is followed by a Sigmoid function to produce a real number in the range of $(0,1)$. 
The network modeling the generator (policy) is a fully connected network with 3 hidden layers of size 64, each followed by a $tanh$ layer. 
The output layer takes the output of the last hidden layer and maps it linearly into the parameters of a Gaussian distribution from which actions can be sampled.

We perform two sets of experiments, one set uses the $L_2$ norm and the other uses the $L_\infty$ norm throughout the experiments.
The norms are used for the following: 1) Defining the balls in which we find the adversarial perturbations $\delta_s$ and $\delta_{s,a}$ through \eqref{opt:adv-pert-disc}, \eqref{opt:adv-pert-gen} in the main text. 2) Defining the ball from which we sample the noise injected at test time. 3) Computing the empirical local-Lipschitz constant (ELLC) of the generator $\pi_\theta$ at a given radius $r_p$ which is calculated as: 
\[ \text{ELLC}(\pi_\theta, r_p)=  \mathbb{E}_{s, \norm{\delta} = r_p} \left[D_J(\pi_\theta(s)||\pi_\theta(s+\delta) / r_p \right] \]
in which $\norm{\delta}$ can be measured with respect to the $L_2$ norm or the $L_\infty$ norm.

The experiments mentioned in the main text correspond to the case where we use the $L_2$ norm for all the above quantities. 
In this section in the appendix, however, we discuss the results when we use the $L_\infty$ instead.

For this set of experiments, we use the same environments as the main text, i.e., simulated robot locomotion environments in the MuJoCo suite \cite{todorov2012mujoco}, namely Walker2d, Hopper, and HalfCheetah.

\textbf{Regularizing the Discriminator.}
First, we investigate the effect of inducing local Lipschitzness in the discriminator of LL-GAIL using $L_\infty$ norm
(see Figures~\ref{subplot:Walker2d-d-L-inf}, \ref{subplot:Hopper-d-L-inf},  \ref{subplot:HalfCheetah-d-L-inf}). 
We refer to this method as LLD-GAIL.
Using Algorithm.~\ref{alg:lip-disc}, we train LLD-GAIL on a range of hyper-parameters $\gamma$ and $r_p$.

Figure~\ref{subplot:Walker2d-d-L-inf} benchmarks the discriminator-regularized LL-GAIL (LLD-GAIL) against the baselines for the Walker2d environment. 
The plot on the left-hand side of Figure~\ref{subplot:Walker2d-d-L-inf} compares the performance of the generators learned by different models in test scenarios where the observations are corrupted by different levels of noise. 
To simulate the effect of noise on observations, we add noise within an $L_\infty$ norm equal to the noise level reported on the plots. 

To understand the link between the local Lipschitzness of the generators and their robustness, the plot on the right-hand side of Figure~\ref{subplot:Walker2d-d-L-inf} displays the empirical local Lipschitzness constant (ELLC) of the trained generators. 
The ELLC is an empirical metric that quantifies the local Lipschitzness properties of a trained generator. 
A generator with a higher EELC is more sensitive to perturbations in the observations.

To train the noisy GAIL baseline, we perturb the observations fed to the discriminator with random noise within an $L_\infty$ ball where the radius of the ball specifies the noise level. 
We train the noisy GAIL with a range of noise levels but we only report the results corresponding to the best training noise level ($0.03$ noise measured in $ L_\infty$ for all cases ). 

Figure~\ref{subplot:Walker2d-d-L-inf} demonstrates that the proposed discriminator-regularization helps LLD-GAIL outperform the baselines by improving the Lipschitzness properties of the generator and in turn those of the imitation policies. 
Figure~\ref{subplot:Walker2d-d-L-inf} further shows that the generators that have a smaller ELLC at a given radius, perform better at a noise level comparable to that radius than generators that have a larger ELLC. The same pattern is observed in Figure.~\ref{subplot:Hopper-d}.
However, we did not see improved robustness for the generator learned through LLD-GAIL or the noisy baseline when using $L_\infty$ norm to regularize the discriminator (Figure. \ref{subplot:HalfCheetah-d-L-inf}).

\textbf{Regularizing the Generator.}
Next, we focus on the effect of directly inducing local Lipschitzness using $L_\infty$ norm in the generator of LL-GAIL on the robustness of the generator (see Figures.~\ref{subplot:Walker2d-g-L-inf}, \ref{subplot:Hopper-g-L-inf},  \ref{subplot:HalfCheetah-g-L-inf}). 
We refer to this generator-regularized LL-GAIL method as LLG-GAIL.
We train the LLG-GAIL according to Algorithm.~\ref{alg:lip-gen} with $L_\infty$ norm. 

When training the noisy GAIL baseline for this set of experiments, we add randomly sampled noise an an $L_\infty$ norm equal to the reported noise level on the plots. 
We report the best results we obtained for the noisy GAIL baseline ($0.3$ noise measured in $ L_\infty$ for all cases ). 

Figure~\ref{subplot:Walker2d-g-L-inf} shows that LLG-GAIL with a regularized generator significantly outperforms both natural GAIL and noisy GAIL across a wide range of noise levels. 
Comparing the left-hand plot and the right-hand plot of Figure~\ref{subplot:Walker2d-g-L-inf}, we observe that the generators with smaller ELLCs vastly outperform those with larger ELLCs, especially at larger noise levels. 
A similar pattern is observed in Figures.~\ref{subplot:Hopper-g-L-inf}, \ref{subplot:HalfCheetah-g-L-inf} which correspond to the Hopper and HalfCheetah environments respectively.

\begin{figure}
\centering
\subfigure[\footnotesize Walker2d, reg.  disc.]{\label{subplot:Walker2d-d-L-inf} \includegraphics[width=0.32\textwidth]{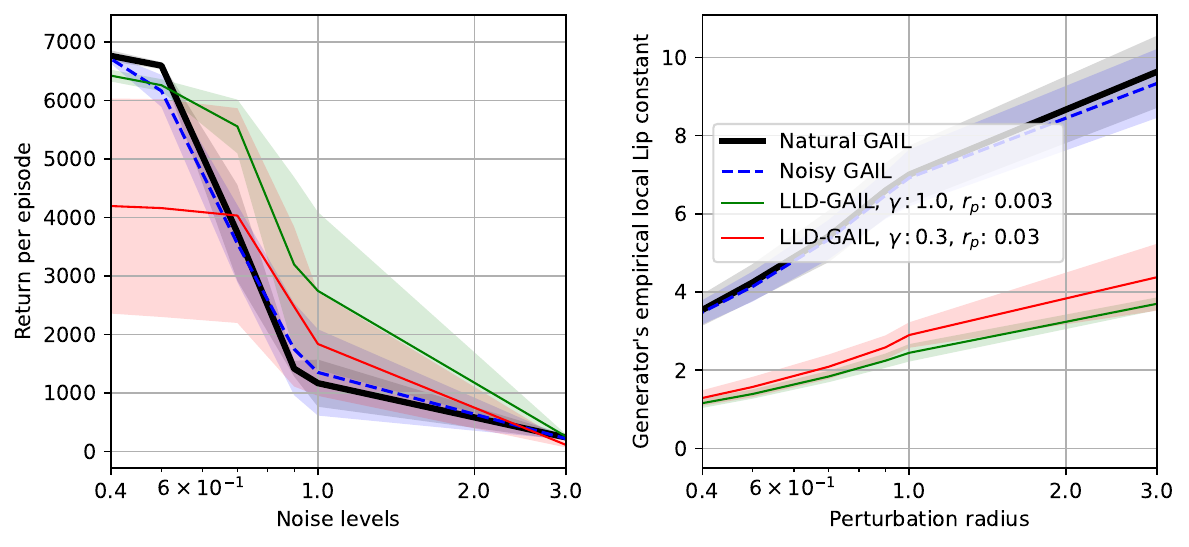}}
\hfill
\subfigure[\footnotesize Walker2d, reg. gen.]{\label{subplot:Walker2d-g-L-inf} \includegraphics[width=0.32\textwidth]{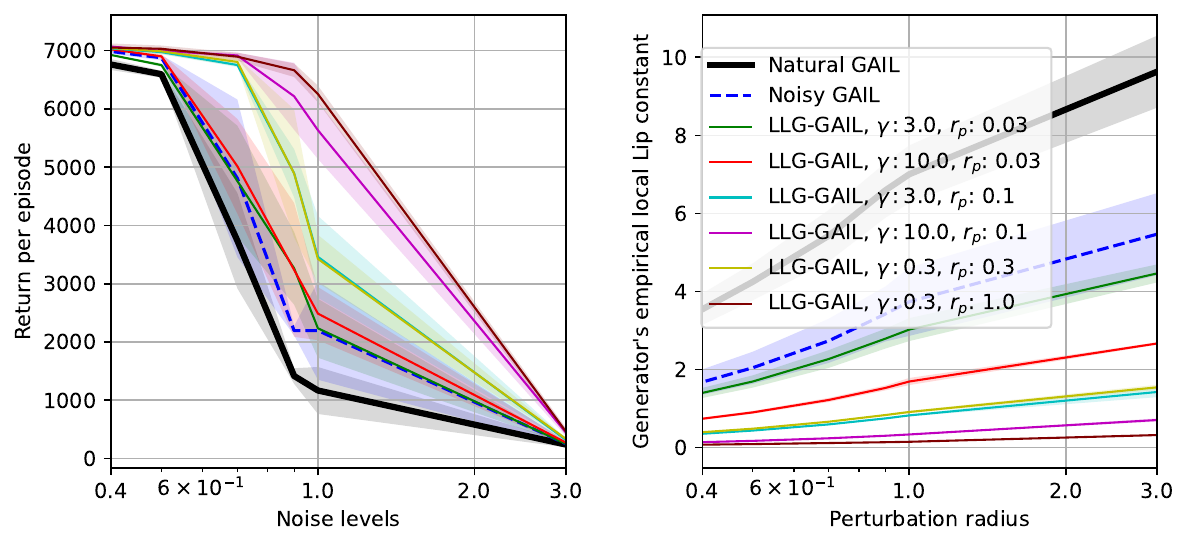}}%
\hfill
\subfigure[\footnotesize Hopper, reg. disc.]{\label{subplot:Hopper-d-L-inf} \includegraphics[width=0.32\textwidth]{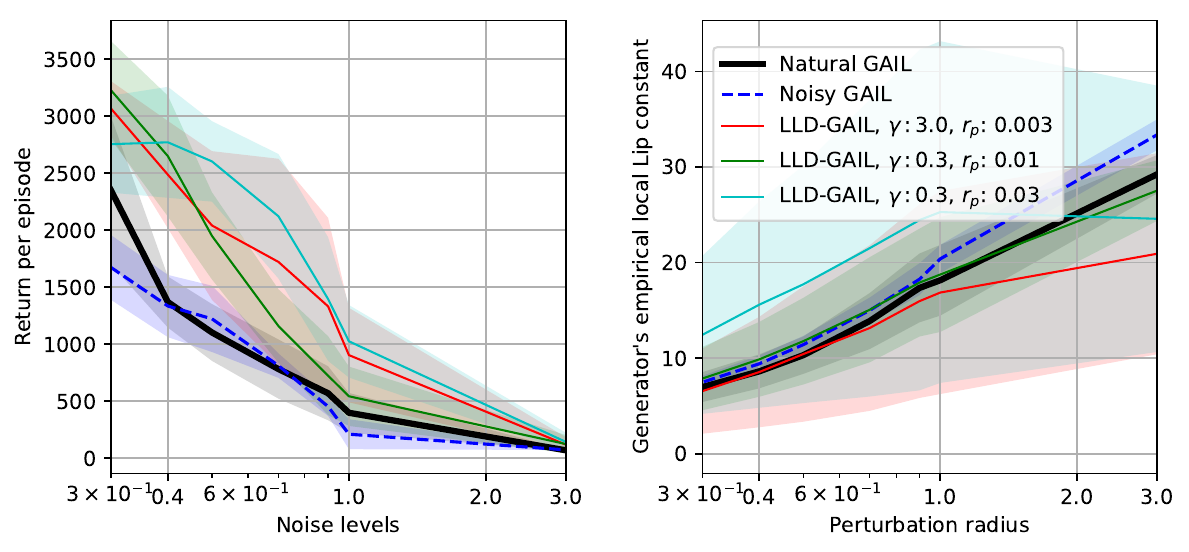}}%
\hfill
\subfigure[\footnotesize Hopper, reg. gen.]{\label{subplot:Hopper-g-L-inf} \includegraphics[width=0.32\textwidth]{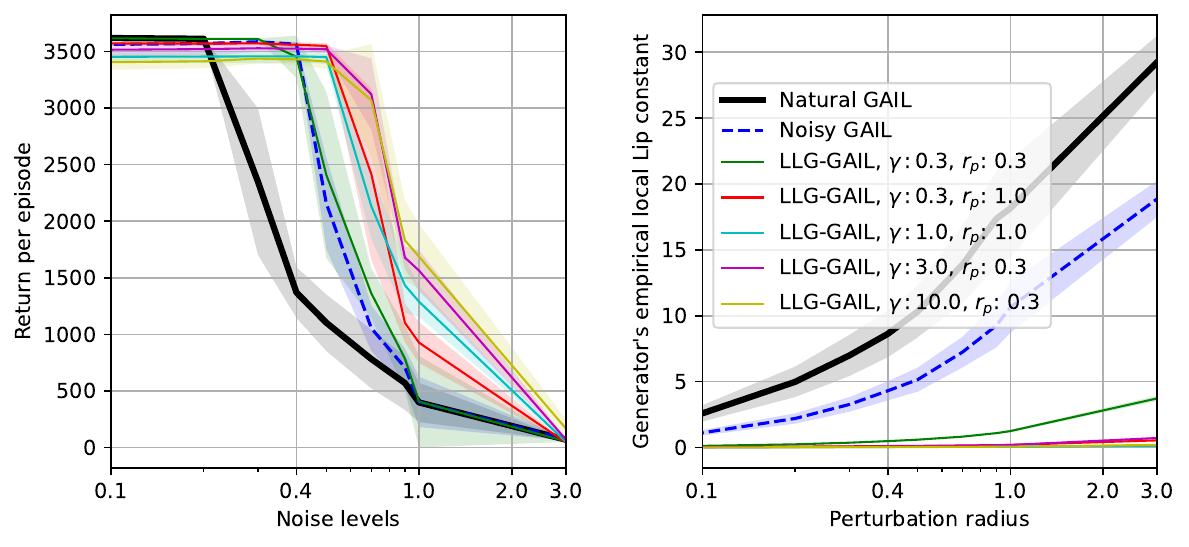}}%
\hfill
\subfigure[\footnotesize HalfCheetah, reg. disc.]{\label{subplot:HalfCheetah-d-L-inf}\includegraphics[width=0.32\textwidth]{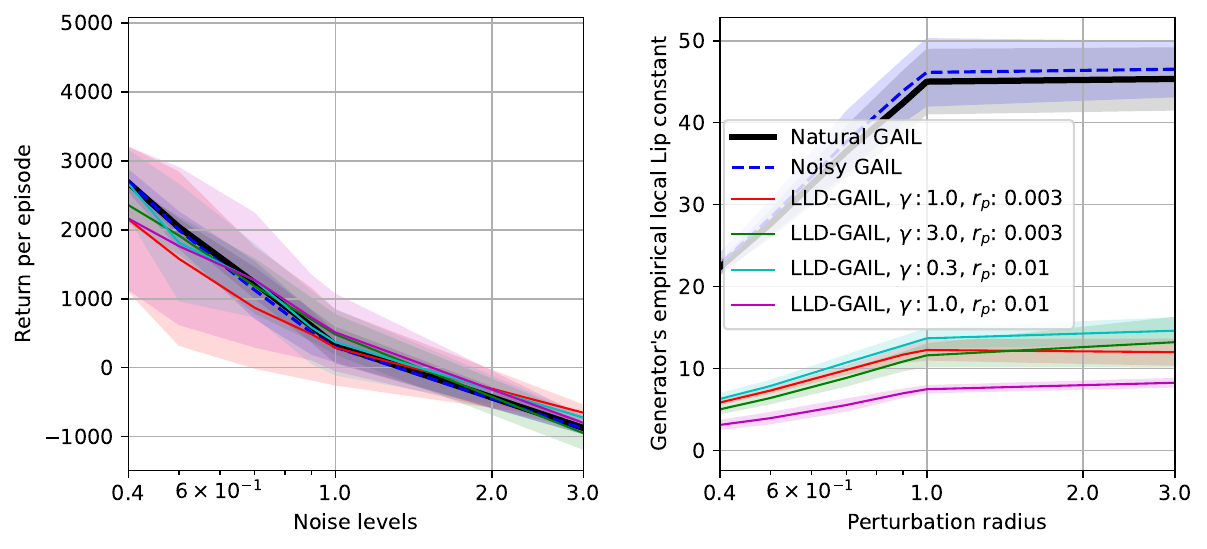}}%
\hfill
\subfigure[\footnotesize HalfCheetah, reg. gen.]{\label{subplot:HalfCheetah-g-L-inf} \includegraphics[width=0.32\textwidth]{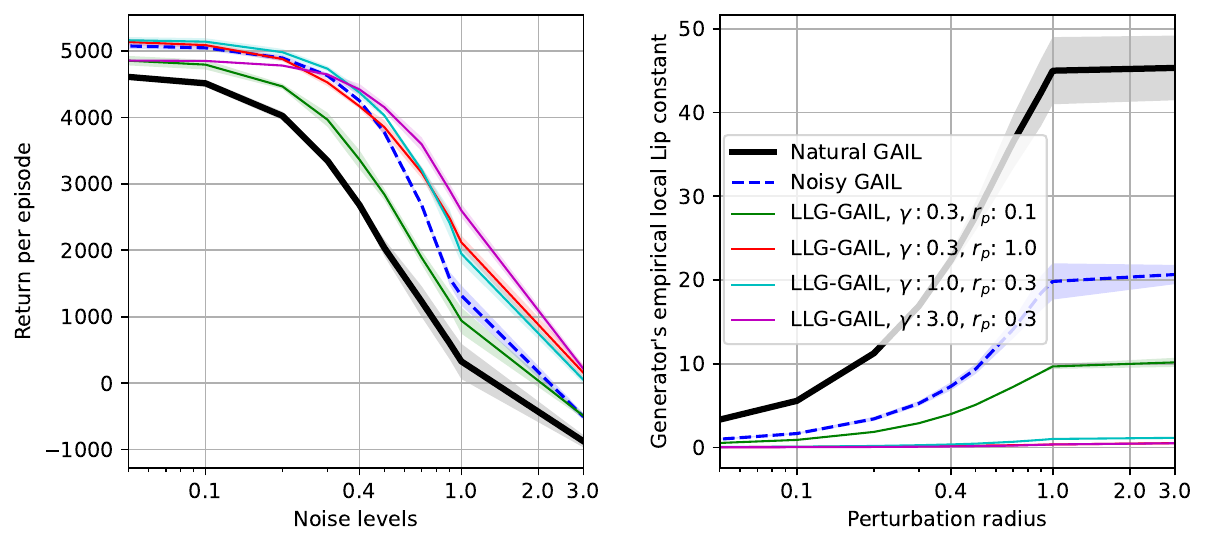}}%
\caption{\footnotesize The comparison between LL-GAIL and the benchmarking schemes natural GAIL and noisy GAIL on several simulated robot locomotion environments in the MuJoCo suite \cite{todorov2012mujoco} when using $L_\infty$ norm.}
\label{fig:LL-GAIL-L-inf}
\end{figure}

\begin{figure}
\centering
\subfigure[\footnotesize Walker2d]{\label{subplot:Walker2d-best} \includegraphics[width=0.24\textwidth]{pics/Walker2d-compare-the-best-L_2.pdf}}
\hfill
\subfigure[\footnotesize Hopper]{\label{subplot:Hopper-best} \includegraphics[width=0.24\textwidth]{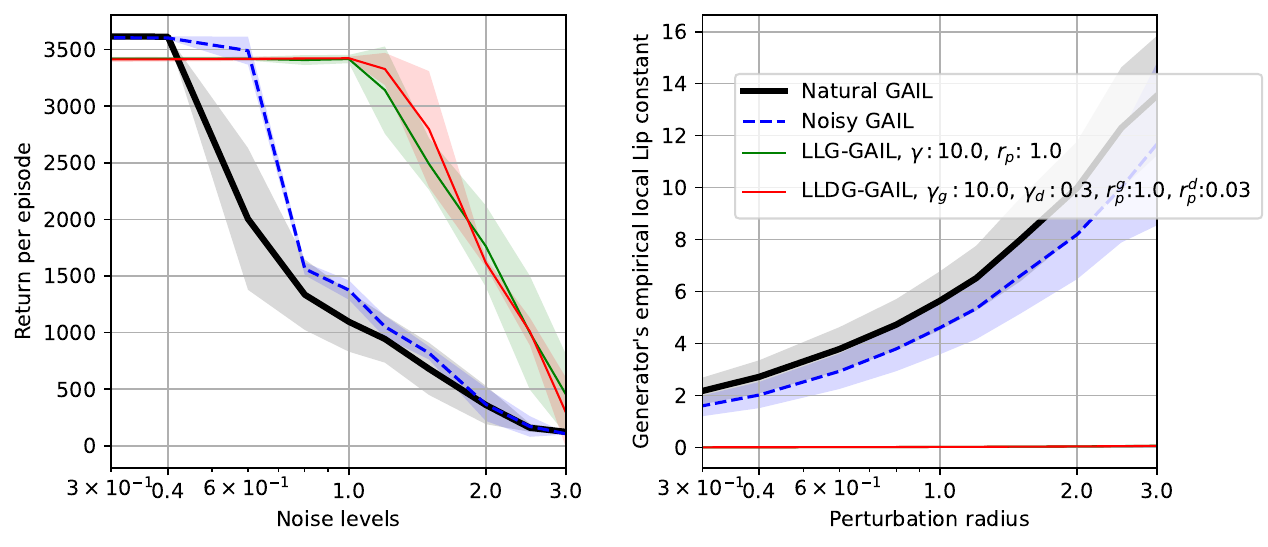}}%
\hfill
\subfigure[\footnotesize HalfCheetah]{\label{subplot:HalfCheetah-best}\includegraphics[width=0.24\textwidth]{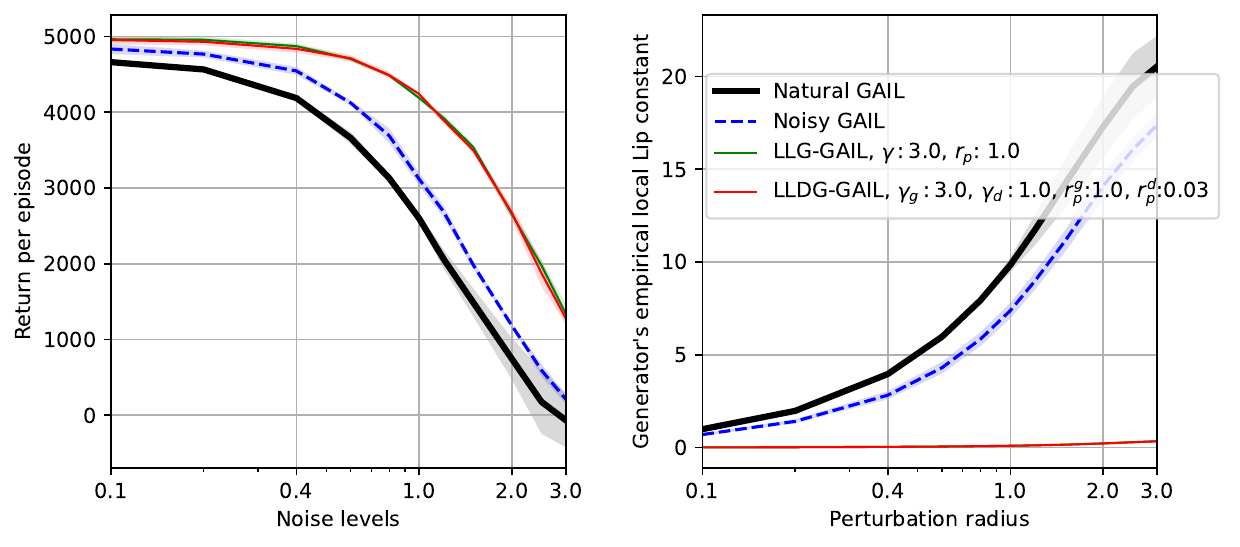}}%
\hfill
\subfigure[\footnotesize Ant]{\label{subplot:Ant-best} \includegraphics[width=0.24\textwidth]{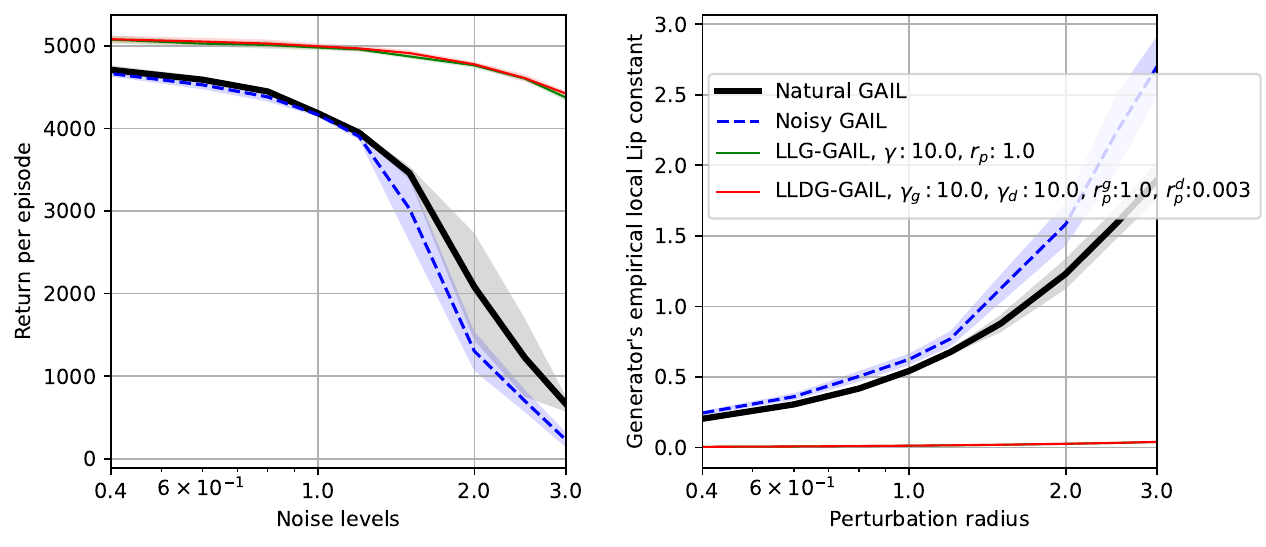}}%
\caption{\footnotesize LLDG-GAIL with both generator and discriminator regularized compared with LLG-GAIL with only the generator regularized and noisy GAIL and natural GAIL. $\gamma_g$ and $r^g_p$ are the hyper-parameters for regularizing the generator of LLDG-GAIL, and $\gamma_d$ and $r^d_p$ are the hyper-parameters for regularizing the discriminator of LLDG-GAIL.}
\label{fig:best-of-all}
\end{figure}

    


\textbf{Regularizing both the generator and discriminator}
Moreover, we investigate the effect of regularizing both the generator and discriminator (Figure.~\ref{fig:best-of-all}) using $L_2$ norm. 
We refer to this method as LLDG-GAIL.
In this set of experiments, we compare the LLDG-GAIL, to LLG-GAIL (which is generator-regularized LL-GAIL), noisy GAIL, and natural GAIL. 
In the parameters $\gamma_g, \gamma_d, r^g_p, r^d_p$ mentioned for LLDG-GAIL, we have used superscripts or subscripts $d$ and $g$ to refer to the hyper-parameters of discriminator regularization and generator regularization respectively. 
Due to the limitations of computational resources, we do not do an exhaustive hyper-parameter tuning for LLDG-GAIL. 
The grid search over the range of hyper-parameters for this case is beyond the computational resources available to us. 

Figure.~\ref{fig:best-of-all} shows that we attain improved robustness of the policy when regularizing both the generator and discriminator (LLDG-GAIL) and this effect is more pronounced for the Walker2d and Hopper environment. However, we observe diminishing returns compared to only regularizing the generator (LLG-GAIL).  
We identify two possible explanations for this phenomenon. First, both of the proposed regularizers aim to promote robustness through inducing local Lipschitzness, and given this common goal, observing a saturated improvement is expected. Additionally, the absence of hyperparameter tuning contributes to such a phenomenon. 
